\documentclass{article}

\usepackage{microtype}
\usepackage{graphicx}
\usepackage{subfigure}
\usepackage{booktabs}
\usepackage{hyperref}

\usepackage{bm}
\usepackage{amsfonts}
\usepackage{todonotes}
\usepackage{amsthm}
\usepackage{amsmath}
\usepackage{algorithm}
\usepackage{algorithmic}
\usepackage{mathtools}
\usepackage{array}
\usepackage{amssymb}
\usepackage{bbold}
\usepackage{thm-restate}
\usepackage{relsize}
\usepackage{multicol}

\newtheorem{definition}{Definition}
\newtheorem{assumption}{Assumption}

\newtheorem{corollary}{Corollary}

\newtheorem{lemma}{Lemma}

\newcommand{\Var}{\mathbb{V}\mathrm{ar}}
\newcommand{\E}{\mathbb{E}}
\newcommand{\R}{\mathbb{R}}

\newcommand{\M}{\mathcal{M}}
\newcommand{\expec}[1]{\mathbb{E}\left[#1\right]}
\newcommand{\prob}[1]{\mathbb{P}\left\{#1\right\}}
\newcommand{\indi}[1]{\mathbb{1}\left\{#1\right\}}

\newcommand{\minsv}[1]{\lambda_{\mathrm{min}}(#1)}
\newcommand{\maxsv}[1]{\lambda_{\mathrm{max}}(#1)}

\newcommand{\devr}[2]{\Delta_{#1}^r(#2)}
\newcommand{\devp}[3]{\Delta_{#1}^p(#2)}

\newcommand{\adevr}[2]{\tilde{\Delta}_{#1}^r(#2)}
\newcommand{\adevp}[3]{\tilde{\Delta}_{#1}^p(#2)}

\definecolor{airforceblue}{rgb}{0.36, 0.54, 0.66}
\definecolor{darkmidnightblue}{rgb}{0.0, 0.2, 0.4}
\definecolor{midnightblue}{rgb}{0.1, 0.1, 0.44}
\newcommand{\hl}[1]{{\color{darkmidnightblue} #1}}

\newcommand{\tr}{^{\mathsmaller T}}

\newcommand{\argmin}{\operatornamewithlimits{argmin}}
\newcommand{\argmax}{\operatornamewithlimits{argmax}}

\renewcommand{\Psi}{\mathcal{I}}

\renewcommand{\tilde}{\widetilde}
\renewcommand{\hat}{\widehat}

\usepackage[accepted]{icml2020}

\icmltitlerunning{Sequential Transfer in Reinforcement Learning with a Generative Model}

\begin{document}

\twocolumn[
\icmltitle{Sequential Transfer in Reinforcement Learning with a Generative Model}




\begin{icmlauthorlist}
\icmlauthor{Andrea Tirinzoni}{poli}
\icmlauthor{Riccardo Poiani}{poli}
\icmlauthor{Marcello Restelli}{poli}
\end{icmlauthorlist}

\icmlaffiliation{poli}{Politecnico di Milano, Milan, Italy}

\icmlcorrespondingauthor{Andrea Tirinzoni}{andrea.tirinzoni@polimi.it}

\icmlkeywords{Reinforcement Learning, Transfer Learning, ICML}

\vskip 0.3in
]



\printAffiliationsAndNotice{}  

\begin{abstract}
We are interested in how to design reinforcement learning agents that provably reduce the sample complexity for learning new tasks by transferring knowledge from previously-solved ones. The availability of solutions to related problems poses a fundamental trade-off: whether to seek policies that are expected to achieve high (yet sub-optimal) performance in the new task immediately or whether to seek information to quickly identify an optimal solution, potentially at the cost of poor initial behavior. In this work, we focus on the second objective when the agent has access to a generative model of state-action pairs. First, given a set of solved tasks containing an approximation of the target one, we design an algorithm that quickly identifies an accurate solution by seeking the state-action pairs that are most informative for this purpose. We derive PAC bounds on its sample complexity which clearly demonstrate the benefits of using this kind of prior knowledge. Then, we show how to learn these approximate tasks sequentially by reducing our transfer setting to a hidden Markov model and employing spectral methods to recover its parameters. Finally, we empirically verify our theoretical findings in simple simulated domains.
\end{abstract}

\section{Introduction}\label{sec:intro}

Knowledge transfer has proven to be a fundamental tool for enabling lifelong reinforcement learning (RL)~\cite{sutton1998reinforcement}, where agents face sequences of related tasks. In this context, upon receiving a new task, the agent aims at reusing knowledge from the previously-solved ones to speed-up the learning process. This has the potential to achieve significant performance improvements over standard RL from scratch, and several studies have confirmed this hypothesis~\cite{taylor2009transfer, lazaric2012transfer}.

A key question is what and how knowledge should be transferred~\cite{taylor2009transfer}. As for the kind of knowledge to be reused, a variety of algorithms have been proposed to transfer experience samples~\cite{lazaric2008transfer, taylor2008transferring, yin2017knowledge, tirinzoni2018importance,tirinzoni2019transfer}, policies \cite{fernandez2006probabilistic, mahmud2013clustering, rosman2016bayesian, abel2018policy, yang2019single, feng2019does}, value-functions \cite{liu2006value,tirinzoni2018transfer}, features \cite{konidaris2012transfer,barreto2017successor, barreto2018transfer}, and more. Regarding how knowledge should be transferred, the answer is more subtle and depends directly on the desired objectives. On the one hand, one could aim at maximizing the \emph{jumpstart}, i.e., the expected initial performance in the target task. This can be done by transferring \emph{initializers}, e.g., policies or value-functions that, based on past knowledge, are expected to immediately yield high rewards. 
Among the theoretical studies, \citet{mann2012directed} and \citet{abel2018policy} showed that jumpstart policies can provably reduce the number of samples needed for learning a target task.  However, since these initializers are computed before receiving the new task, the agent, though starting from good performance, might still take a long time before converging to near-optimal behavior. An alternative approach is to spend some initial interactions with the target to gather information about the task itself, so as to better decide what to transfer. For instance, the agent could aim at identifying which of the previously-solved problems is most similar to the target. If the similarity is actually large, transferring the past solution would instantaneously lead to near-optimal behavior. 
This \emph{task identification} problem has been studied by~\citet{dyagilev2008efficient, brunskill2013sample, liu2016pac}. One downside is that these approaches do not actively seek information to minimize identification time (or, equivalently, sample complexity), in part because it is non-trivial to find which actions are the most informative given knowledge from related tasks. To the best of our knowledge, it is an open problem how to leverage this prior knowledge to reduce identification time.


Regardless of how knowledge is transferred, a common assumption is that tasks are drawn from some fixed distribution~\cite{taylor2009transfer}, which often hides the sequential nature of the problem. In many lifelong learning problems, tasks evolve dynamically and are temporally correlated. Consider an autonomous car driving through traffic. Based on the traffic conditions, both the environment dynamics and the agent's desiderata could significantly change. However, this change follows certain temporal patterns, e.g., the traffic level might be very high in the early morning, medium in the afternoon, and so on. This setting could thus be modeled as a sequence of related RL problems with temporal dependencies. It is therefore natural to expect that a good transfer agent will exploit the sequential structure among tasks rather than only their static similarities.

Our paper aims at advancing the theory of transfer in RL by addressing the following questions related to the above-mentioned points: (1) how can the agent quickly identify an accurate solution of the target task when provided with solved related problems? (2) how can the agent exploit the sequential nature of the lifelong learning setup? 

Taking inspiration from previous works on non-stationary RL~\cite{choi2000environment, hadoux2014solving}, we model the sequential transfer setting by assuming that the agent faces a finite number of tasks whose evolution is governed by an underlying Markov chain. Each task has potentially different rewards and dynamics, while the state-action space is fixed and shared. The motivation for this framework is that some real systems often work in a small number of operating \emph{modes} (e.g., the traffic condition in the example above), each with different dynamics. Unlike these works, we assume that the agent is informed when a new task arrives as in the standard transfer setting~\cite{taylor2009transfer}. Then, as in previous studies on transfer in RL~\cite{brunskill2013sample, azar2013sequential}, we decompose the problem into two parts. First, in Section~\ref{sec:transfer}, we assume that the agent is provided with prior knowledge in the form of a set of tasks containing an approximation to the target one. Under the assumption that a generative model of the environment is available, we design an algorithm that actively seeks information in order to identify a near-optimal policy and transfers it to the target task. We derive PAC bounds~\cite{strehl2009reinforcement} on its sample complexity which are significantly tighter than existing results and clearly demonstrate performance improvements over learning from scratch. Then, in Section~\ref{sec:structure}, we show how this prior knowledge can be learned in our sequential setup to allow knowledge transfer. We reduce the problem to learning a hidden Markov model and use spectral methods~\cite{anandkumar2014tensor} to estimate the task models necessary for running our policy transfer algorithm. We derive finite-sample bounds on the error of these estimated models and discuss how to leverage the temporal correlations to further reduce the sample complexity of our methods. Finally, we report numerical simulations in some standard RL benchmarks which confirm our theoretical findings. In particular, we show examples where identification yields faster convergence than jumpstart methods, and examples where the exploitation of the sequential correlations among tasks is superior than neglecting them.

\section{Preliminaries}\label{sec:prelims}

We model each task $\theta$, as a discounted Markov decision process (MDP) \cite{puterman2014markov}, $\mathcal{M}_\theta := \langle \mathcal{S}, \mathcal{A}, p_\theta, q_\theta, \gamma \rangle$, where $\mathcal{S}$ is a finite set of $S$ states, $\mathcal{A}$ is a finite set of $A$ actions, $p_\theta: \mathcal{S} \times \mathcal{A} \rightarrow \mathcal{P}(\mathcal{S})$ are the transition probabilities, $q_\theta: \mathcal{S} \times \mathcal{A} \rightarrow \mathcal{P}(\mathcal{U})$ is the reward distribution over space $\mathcal{U}$ (e.g., $\mathcal{U} = \mathbb{R}$), and $\gamma \in [0,1)$ is a discount factor. Here $\mathcal{P}$($\Omega$) denotes the set of probability distributions over a set $\Omega$. We suppose that the sets of states and actions are fixed and that tasks are uniquely determined by their parameter $\theta$. For this reason, we shall occasionally use $\theta$ to indicate $\mathcal{M}_\theta$, while alternatively referring to it as task, parameter, or (MDP) model. We denote by $r_\theta(s,a) := \E_{q_\theta}[U_t | S_t=s,A_t=a]$ and assume, without loss of generality, that rewards take values in $[0,1]$. We use $V_\theta^\pi(s)$ to indicate the value function of a (deterministic) policy $\pi : \mathcal{S} \rightarrow \mathcal{A}$ in task $\theta$, i.e., the expected discounted return achieved by $\pi$ when starting from state $s$ in $\theta$, $V_\theta^\pi(s) := \E_\theta^\pi\left[\sum_{t=0}^\infty \gamma^t U_t | S_0 = s\right]$. The optimal policy $\pi^*_\theta$ for task $\theta$ maximizes the value function at all states simultaneously, $V^{\pi^*_\theta}_\theta(s) \geq V^{\pi}_\theta(s)$ for all $s$ and $\pi$. We let $V^*_\theta(s) := V^{\pi^*_\theta}_\theta(s)$ denote the optimal value function of $\theta$. Given $\epsilon > 0$, we say that a policy $\pi$ is $\epsilon$-optimal for $\theta$ if $V^\pi_\theta(s) \geq V^*_\theta(s) - \epsilon$ for all states $s$. We define $\sigma_\theta^r(s,a)^2 := \Var_{q_\theta(\cdot | s,a)}[U]$ as the variance of the reward in $s,a$ for task $\theta$, and $\sigma_\theta^p(s,a;\theta')^2 := \Var_{p_\theta(\cdot | s,a)}[V_{\theta'}^*(S')]$ as the variance of the optimal value function of $\theta'$ under the transition model of $\theta$. To simplify the exposition, we shall alternatively use the standard vector notation to indicate these quantities. For instance, $V^*_\theta \in \R^S$ is the vector of optimal values, $p_\theta(s,a) \in \R^S$ is the vector of transition probabilities from $s,a$, $r_\theta \in \R^{SA}$ is the flattened reward matrix, and so on. To measure the distance between two models $\theta,\theta'$, we define $\devr{s,a}{\theta,\theta'} := |{r}_\theta(s,a) - {r}_{\theta'}(s,a)|$ for the rewards and $\devp{s,a}{\theta,\theta'}{} = |({p}_\theta(s,a) - {p}_{\theta'}(s,a))\tr V^*_\theta|$ for the transition probabilities. The latter measures how much the expected return of an agent taking $s,a$ and acting optimally in $\theta$ changes when the first transition only is governed by $\theta'$. See Appendix \ref{app:notation} for a quick reference of notation.

\paragraph{Sequential Transfer Setting}

We model our problem as a hidden-mode MDP \cite{choi2000hidden}. We suppose that the agent faces a finite number of $k$ possible tasks, $\Theta = \{\theta_1, \theta_2, \dots, \theta_k\}$. Tasks arrive in sequence and evolve according to a Markov chain with transition matrix $T \in [0,1]^{k\times k}$ and an arbitrary initial task distribution. Let $\theta^*_h$, for $h=1,\dots,m$, be the $h$-th random task faced by the agent. Then, $[T]_{i,j} = \prob{\theta^*_{h+1} = \theta_i | \theta^*_h = \theta_j}$. The agent only knows the number of tasks $k$, while both the MDP models and the task-transition matrix $T$ are unknown. The goal is to identify an $\epsilon$-optimal policy for $\theta^*_h$ using as few samples as possible, while leveraging knowledge from the previous tasks $\theta^*_1, \dots, \theta^*_{h-1}$ to further reduce the sample complexity. In order to provide deeper insights into how to design efficient identification strategies and facilitate the analysis, we assume that the agent can access a generative model of state-action pairs. Similarly to experimental optimal design~\cite{pukelsheim2006optimal}, upon receiving each new task $\theta^*_h$, the agent can perform at most $n$ \emph{experiments} for identification, where each experiment consists in choosing an arbitrary state-action pair $s,a$ and receiving a random next-state $S' \sim p_{\theta^*_h}(\cdot | s,a)$ and reward $R \sim q_{\theta^*_h}(\cdot | s,a)$. 
After this \emph{identification phase}, the agent has to transfer a policy to the target task and starts interacting with the environment in the standard online fashion.

As in prior works \citep[e.g.,][]{brunskill2013sample, azar2013sequential, liu2016pac}, we suppose that the agent maintains an approximation to all MDP models, $\{\tilde{\M}_\theta\}_{\theta\in\Theta}$, and decompose the problem in two main steps.\footnote{In the remainder, we overload the notation by using tildes to indicate quantities related to approximate models.} (1) First, in Section \ref{sec:transfer} we present an algorithm that, given the approximate models together with the corresponding approximation errors, actively queries the generative model of the target task by seeking state-action pairs that yield high information about the optimal policy of $\theta^*_h$. (2) Then, in Section \ref{sec:structure} we show how the agent can build these approximate models from the samples collected while interacting with the sequence of tasks so far. 


\section{Policy Transfer from Uncertain Models}\label{sec:transfer}

Throughout this section, we shall focus on learning a single model $\theta^*\in\Theta$ given knowledge from previous tasks. Therefore, to simplify the notation, we shall drop dependency on the specific task sequence (i.e., we drop the task indexes $h$).

Let $\{\tilde{\M}_\theta\}_{\theta\in\Theta}$ be the approximate MDP models, with $\tilde{p}_\theta$ and $\tilde{r}_\theta$ denoting their transition probabilities and rewards, respectively. We use the intuition that, if these approximations are accurate enough, the problem can be reduced to identifying a suitable policy among the optimal ones of the MDPs $\{\tilde{\M}_\theta\}_{\theta\in\Theta}$, which necessarily contains an approximation of $\pi^*_{\theta^*}$.
We rely on the following assumption to assess the quality of the approximate models.
\begin{assumption}[Model uncertainty]\label{ass:model-err}
There exist known constants $\Delta_{\mathrm{max}}^r$, $\Delta_{\mathrm{max}}^p$, $\Delta_{\mathrm{max}}^{\sigma_r}$, $\Delta_{\mathrm{max}}^{\sigma_p}$ such that
\begin{align*}
    \max_{\theta\in\Theta} \| r_\theta - \tilde{r}_\theta \|_\infty &\leq \Delta_{\mathrm{max}}^r,\\
    \max_{\theta,\theta' \in \Theta}\|(p_\theta - \tilde{p}_\theta)\tr \tilde{V}^{*}_{\theta'} \|_\infty &\le \Delta_{\mathrm{max}}^p,\\
    \max_{\theta\in\Theta}\| \sigma^{r}_\theta - \tilde{\sigma}^{r}_\theta\|_\infty &\le \Delta_{\mathrm{max}}^{\sigma_r},\\
    \max_{\theta,\theta'\in\Theta} \| \sigma^{p}_\theta(\theta') - \tilde{\sigma}_\theta^{p}(\theta') \|_\infty &\le \Delta_{\mathrm{max}}^{\sigma_p}.
\end{align*}
\end{assumption}
Assumption~\ref{ass:model-err} ensures that an upper bound to the approximation error of the given models is known. In particular, the maximum error among all components, $\Delta := \max\{\Delta_{\mathrm{max}}^r, \Delta_{\mathrm{max}}^p, \Delta_{\mathrm{max}}^{\sigma_r}, \Delta_{\mathrm{max}}^{\sigma_p} \}$ is a fundamental parameter for our approach. In Section~\ref{sec:structure}, we shall see how to guarantee this assumption when facing tasks sequentially.

\subsection{The PTUM Algorithm}

\begin{algorithm}[t]
\small
\caption{\small Policy Transfer from Uncertain Models (PTUM)} \label{alg:policy-transfer}
\begin{algorithmic}[1]
\REQUIRE Set of approximate MDPs $\{\tilde{\M}_\theta\}_{\theta\in\Theta}$, accuracy $\epsilon$, confidence $\delta$, number of samples $n$, model uncertainty $\Delta$
\ENSURE An $\epsilon$-optimal policy for $\M_{\theta^*}$ with probability $1-\delta$
\vspace{0.1cm}
\STATE{\hl{\textsc{// Check accuracy condition}}}
\STATE{If $\Delta \geq \frac{\epsilon(1-\gamma)}{4(1+\gamma)}$ $\rightarrow$ Stop and run $(\epsilon,\delta)$-PAC algorithm}
\STATE{\hl{\textsc{// Transfer Mode}}}
\STATE{Initialize datasets $\mathcal{D}_{s,a}^r, \mathcal{D}_{s,a}^p \leftarrow \emptyset$}
\FOR{$t=1,2,\dots, n$}
\STATE{\hl{\textsc{// Step 1. Build empirical MDP model}}}
\STATE{$\hat{r}_t(s,a) \leftarrow \frac{1}{N_t(s,a)}\sum_{u \in \mathcal{D}^r_{s,a}} u$}
\STATE{$\hat{p}_t(s' | s,a) \leftarrow \frac{1}{N_t(s,a)}\sum_{s'' \in \mathcal{D}^p_{s,a}} \indi{s' = s''}$}
\STATE{$\hat{\sigma}_t^r(s,a)^2 \leftarrow \frac{\sum_{u \in \mathcal{D}^r_{s,a}} (u - \hat{r}_t(s,a))^2}{N_t(s,a) - 1}$}
\STATE{$\hat{\sigma}_t^p(s,a;\theta')^2 \leftarrow \frac{\sum_{s'' \in \mathcal{D}^p_{s,a}} (\tilde{V}^*_{\theta'}(s'') - \hat{p}_t(s,a)\tr \tilde{V}^*_{\theta'})^2}{N_t(s,a) - 1}$}
\STATE{\hl{\textsc{// Step 2. Update confidence set}}}
\STATE{$\bar{\Theta}_{t} \leftarrow \left\{ \vphantom{\Big|}\theta\in\bar{\Theta}_{t-1}\ \big|\ \text{\eqref{eq:single-conf-r}-\eqref{eq:single-conf-sp} hold for all $s,a$ and $\theta'\in\Theta$} \right\}$}
\STATE{\hl{\textsc{// Step 3. Check stopping condition}}}
\STATE{If there exists $\theta\in\bar{\Theta}_t$ such that for all $\theta' \in \bar{\Theta}_t$ we have $\tilde{V}_{\theta'}^{\tilde{\pi}_\theta^*} \geq \tilde{V}_{\theta'}^* - \epsilon + \frac{2\Delta(1+\gamma)}{1-\gamma}$ $\rightarrow$ Stop and return $\tilde{\pi}_\theta^*$}
\STATE{\hl{\textsc{// Step 4. Query generative model}}}
\STATE{$\Psi_t^r(s,a) \leftarrow \max_{\theta,\theta'\in\bar{\Theta}_t} \Psi_{s,a}^r(\theta,\theta')$}
\STATE{$\Psi_t^p(s,a) \leftarrow \max_{\theta,\theta'\in\bar{\Theta}_t} \Psi_{s,a}^p(\theta,\theta')$}
\STATE{$(S_t,A_t) \leftarrow \argmax_{s,a} \max \left\{  \Psi_t^r(s,a), \Psi_t^p(s,a)\right\}$}
\STATE{Obtain $S_{t+1} \sim p_{\theta^*}(\cdot | S_t,A_t)$ and $U_{t+1} \sim q_{\theta^*}(\cdot | S_t,A_t)$}
\STATE{Store $\mathcal{D}^p_{s,a} = \mathcal{D}^p_{s,a} \cup \{S_{t+1}\}$ and $\mathcal{D}^r_{s,a} = \mathcal{D}^r_{s,a} \cup \{U_{t+1}\}$}
\ENDFOR
\STATE{If the algorithm did not stop $\rightarrow$ Run $(\epsilon,\delta)$-PAC algorithm}
\end{algorithmic}
\end{algorithm}

We now present Policy Transfer from Uncertain Models (PTUM), whose pseudo-code is provided in Algorithm~\ref{alg:policy-transfer}. Given the approximate models $\{\tilde{\M}_\theta\}_{\theta\in\Theta}$, whose maximum error is bounded by $\Delta$ from Assumption \ref{ass:model-err}, and two values $\epsilon,\delta > 0$, our approach returns a policy which, with probability at least $1-\delta$, is $\epsilon$-optimal for the target task $\theta^*$. We now describe in detail all the steps of Algorithm~\ref{alg:policy-transfer}.

First (lines 1-2), we check whether positive transfer can occur. The intuition is that, if the approximate models are too uncertain (i.e., $\Delta$ is large), the transfer of a policy might actually lead to poor performance in the target task, i.e., \emph{negative transfer} occurs. Our algorithm checks whether $\Delta$ is below a certain threshold (line 2). Otherwise, we run any $(\epsilon, \delta)$-PAC\footnote{An algorithm is $(\epsilon,\delta)$-PAC if, with probability $1-\delta$, it computes an $\epsilon$-optimal policy using a polynomial number of samples in the relevant problem-dependent quantities~\cite{strehl2009reinforcement}.} algorithm to obtain an $\epsilon$-optimal policy. 
Later on, we shall discuss how this algorithm could be chosen. Although the condition at line 2 seems restrictive, as $\Delta$ is required to be below a factor of $\epsilon$, we conjecture this dependency to be nearly-tight (at least in a worst-case sense). In fact, \citet{feng2019does} have recently shown that the sole knowledge of a poorly-approximate model cannot reduce the worst-case sample complexity of any agent seeking an $\epsilon$-optimal policy. If the condition at line 2 fails, i.e., the models are accurate enough, we say that the algorithm enters the \emph{transfer mode}. Here, the generative model is queried online until an $\epsilon$-optimal policy is found. Similarly to existing works on model identification~\cite{dyagilev2008efficient, brunskill2013sample}, the algorithm proceeds by elimination. At each time-step $t$ (up to at most $n$), we keep a set $\bar{\Theta}_t \subseteq \Theta$ of those models, called \emph{active}, that are likely to be (close approximations of) the target $\theta^*$. This set is created based on the distance between each model and the empirical MDP constructed with the observed samples. Then, the algorithm chooses the next state-action pair $S_t,A_t$ to query the generative model so that the samples from $S_t,A_t$ are informative to eliminate one of the "wrong" models from the active set. This process is iterated until the algorithm finds a policy that is $\epsilon$-optimal for all active models, in which case the algorithm stops and returns such policy. We now describe these main steps in detail.

\paragraph{Step 1. Building the empirical MDP}

In order to find the set of active models $\bar{\Theta}_t$ at time $t$, the algorithm first builds an empirical MDP as a proxy for the true one. Let $N_t(s,a)$ be the number of samples collected from $s,a$ up to (and not including) time $t$. First, the algorithm estimates, for each $s,a$, the empirical rewards $\hat{r}_t(s,a)$ and transition probabilities $\hat{p}_t(s,a)$ (lines 7-8). Then, it computes the empirical variance of the rewards $\hat{\sigma}_t^r(s,a)^2$ and of the optimal value functions $\hat{\sigma}_t^p(s,a;\theta')^2$ for each model $\theta'\in\Theta$ (lines 9-10). This quantities are arbitrarily initialized when $N_t(s,a) = 0$.

\paragraph{Step 2. Building the confidence set}

We define the confidence set $\bar{\Theta}_t$ as the set of models that are ``compatible" with the empirical MDP in \emph{all} steps up to $t$. Formally, a model $\theta \in \Theta$ belongs to the confidence set $\bar{\Theta}_t$ at time $t$ if it was active before (i.e., $\theta\in\bar{\Theta}_{t-1}$) and the following conditions are satisfied for all $s\in\mathcal{S}$, $a\in\mathcal{A}$, and $\theta'\in\Theta$:
\begin{align}
\label{eq:single-conf-r}|\hat{r}_t(s,a) - \tilde{r}_{\theta}(s,a)| &\leq C_{t,\delta}^r(s,a), \\
\label{eq:single-conf-p}|(\hat{p}_t(s,a) - \tilde{p}_\theta(s,a))\tr \tilde{V}^*_{\theta'}| &\leq C_{t,\delta}^p(s,a,\theta'),\\
\label{eq:single-conf-sr}|\hat{\sigma}_t^r(s,a) - \tilde{\sigma}_{\theta}^r(s,a)| &\leq C_{t,\delta}^{\sigma_{r}}(s,a),\\
\label{eq:single-conf-sp}|\hat{\sigma}_t^p(s,a;\theta') - \tilde{\sigma}_{\theta}^p(s,a;\theta')| &\leq C_{t,\delta}^{\sigma_{p}}(s,a).
\end{align}
Intuitively, a model belongs to the confidence set if its distance to the empirical MDP does not exceed, in any component, a suitable confidence level $C_{t,\delta}(s,a)$. The latter has the form $\sqrt{\frac{\log c / \delta}{N_t(s,a)}}$ and is obtained from standard applications of Bernstein's inequality~\cite{boucheron2003concentration}. We refer the reader to Appendix~\ref{app:proof3} for the full expression. Alternatively, we say that a model is \emph{eliminated} from the confidence set (i.e., it will never be active again) as soon as it is not compatible with the empirical MDP. It is important to note that, with probability at least $1-\delta$, the target task $\theta^*$ is never eliminated from $\bar{\Theta}_t$ (see Lemma~\ref{lemma:conf} in Appendix~\ref{app:proof3}).

\paragraph{Step 3. Checking whether to stop}

After building the confidence set, the algorithm checks whether the optimal policy of some active model is $(\epsilon - 2\Delta\frac{1+\gamma}{1-\gamma})$-optimal in all other models in $\bar{\Theta}_t$, in which case it stops and returns this policy. As we shall see in our analysis, this ensures that the returned policy is also $\epsilon$-optimal for $\theta^*$.

\paragraph{Step 4. Deciding where to query the generative model}

The final step involves choosing the next state-action pair $S_t,A_t$ from which to obtain a sample. This is a key point as the sampling rule is what directly determines the sample-complexity of the algorithm. As discussed previously, our algorithm eliminates the models from $\bar{\Theta}_t$ until the stopping condition is verified. Therefore, a good sampling rule should aim at minimizing the stopping time, i.e., it should aim at eliminating as soon as possible all models that prevent the algorithm from stopping. The design of our strategy is driven by this principle. Given the set of active models $\bar{\Theta}_t$, we compute, for each $s,a$, a score $\Psi_t(s,a)$, which we refer to as the \emph{index} of $s,a$, that is directly related to the information to discriminate between any two active models using $s,a$ only (lines 20-22). Then, we choose the $s,a$ that maximizes the index, which can be interpreted as sampling the state-action pair that allows us to discard an active model in the shortest possible time. We confirm this intuition in our analysis later. Formally, our information measure is defined as follows.
\begin{definition}[Information for model discrimination]\label{def:psi}
Let $\theta,\theta' \in \Theta$. For $\tilde{\Delta}_{s,a}^x := [\tilde{\Delta}^x_{s,a}(\theta, \theta') -8\Delta]_+$, with $x \in \{r,p\}$, the information for discriminating between $\theta$ and $\theta'$ using reward/transition samples from $s,a$ are, respectively,
\begin{align*}
\Psi_{s,a}^r(\theta,\theta') &= \min \left\{ \left(\frac{\tilde{\Delta}_{s,a}^r}{\tilde{\sigma}_{\theta}^r(s,a)}\right)^2, \tilde{\Delta}_{s,a}^r \right\}, \\
\Psi_{s,a}^p(\theta,\theta') &= \min \left\{ \left(\frac{\tilde{\Delta}_{s,a}^p}{\tilde{\sigma}_{\theta}^p(s,a;\theta)}\right)^2, (1-\gamma)\tilde{\Delta}_{s,a}^p \right\}.
\end{align*}
The total information is the maximum of these two, $\Psi_{s,a}(\theta,\theta') = \max \left\{ \Psi_{s,a}^r(\theta,\theta'), \Psi_{s,a}^p(\theta,\theta')\right\}$.
\end{definition}
The information $\Psi$ is a fundamental tool for our analysis and it can be understood as follows. The terms on the left-hand side are ratios of the squared deviation between the means of the random variables involved and their variance. If these random variables were Gaussian, this would be proportional to the Kullback-Leibler divergence between the distributions induced by the two models, which in turn is related to their mutual information. The terms on the right-hand side arise from our choice of Bernstein's confidence intervals but have a minor role in the algorithm and its analysis. 

\subsection{Sample Complexity Analysis}

We now analyze the sample complexity of Algorithm~\ref{alg:policy-transfer}. Throughout the analysis, we assume that the model uncertainty $\Delta$ is such that Algorithm~\ref{alg:policy-transfer} enters the transfer mode and that the sample budget $n$ is large enough to allow the identification. The opposite case is of minor importance as it reduces to the analysis of the chosen $(\epsilon,\delta)$-PAC algorithm.
\begin{restatable}{theorem}{sampcomp}\label{th:pac-bound}
Assume $\Delta$ is such that Algorithm \ref{alg:policy-transfer} enters the transfer mode. Let $\tau$ be the random stopping time and $\pi_{\tau}$ be the returned policy. Then, with probability at least $1-\delta$, $\pi_{\tau}$ is $\epsilon$-optimal for $\theta^*$ and the total number of queries to the generative model can be bounded by
\begin{align*}
        \tau \leq \frac{128\min\{SA, |\Theta|\}\log (8SAn(|\Theta|+1)/\delta)}{\max_{s,a}\min_{\theta\in\Theta_\epsilon}\Psi_{s,a}(\theta^*,\theta)},
\end{align*}
where, for $\kappa_\epsilon := \frac{(1-\gamma)\epsilon}{4}-\frac{\Delta(1+\gamma)}{2}$, the set $\Theta_\epsilon \subseteq \Theta$ is
\begin{align*}
\Theta_\epsilon := \left\{\theta \ \Big| \ \| \tilde{r}_\theta - \tilde{r}_{\theta^{*}} \| > \kappa_\epsilon \vee  \| (\tilde{p}_\theta - \tilde{p}_{\theta^{*}})\tr \tilde{V}^*_{\theta^*} \| > \frac{\kappa_\epsilon}{\gamma} \right\}.
\end{align*}
\end{restatable}
The proof, provided in Appendix \ref{app:proof3}, combines standard techniques used to analyze PAC algorithms \citep{azar2013minimax,zanette2019almost} with recent ideas in field of structured multi-armed bandits \citep{tirinzoni2020novel}.
At first glance, Theorem~\ref{th:pac-bound} looks quite different from the standard sample complexity bounds available in the literature~\cite{strehl2009reinforcement}. We shall see now that it reveals many interesting properties. First, this result implies that PTUM is $(\epsilon,\delta)$-PAC as the sample complexity is bounded by polynomial functions of all the relevant quantities. Next, we note that, except for logarithmic terms, the sample complexity scales with the minimum between the number of tasks and the number of state-action pairs. As in practice, we expect the former to be much smaller than the latter, we get a significant gain compared to the no-transfer case, where, even with a generative model, the sample complexity is at least linear in $SA$. The set $\Theta_\epsilon$ can be understood as the set of all models in $\Theta$ whose optimal policy cannot be guaranteed as $\epsilon$-optimal for the target $\theta^*$. As Lemma~\ref{lemma:theta-eps} in Appendix~\ref{app:proof3} shows, it is sufficient to eliminate all models in this set to ensure stopping. Our key result is that the sample complexity of PTUM is proportional to the one of an "oracle" strategy that knows in advance the most informative state-action pairs to achieve this elimination. Note, in fact, that the denominator involves the information to discriminate any model in $\Theta_\epsilon$ with $\theta^*$, but the latter is not known to the algorithm. The following result provides further insights into the improvements over the no-transfer case.
\begin{restatable}{corollary}{corcomp}\label{cor:samp-comp}
Let $\Gamma$ be the minimum gap between $\theta^*$ and any other model in $\Theta$,
\begin{align*}
    \Gamma := \min_{\theta \neq \theta^*} \max \left\{ \| \tilde{r}_\theta - \tilde{r}_{\theta^{*}} \|, \| (\tilde{p}_\theta - \tilde{p}_{\theta^{*}})\tr \tilde{V}^*_{\theta^*} \| \right\}.
\end{align*}
Then, with probability at least $1-\delta$,
\begin{align*}
    \tau \leq \tilde{\mathcal{O}}\left( \frac{\min\{SA, |\Theta|\}\log(1/\delta)}{\max\{\Gamma^2, \epsilon^2\}(1-\gamma)^4} \right).
\end{align*}
\end{restatable}
This result reveals that the sample complexity of Algorithm~\ref{alg:policy-transfer} does not scale with $\epsilon$, which is typically regarded as the main term in PAC bounds. That is, when $\epsilon$ is small, the bound scales with the minimum gap $\Gamma$ between the approximate models. Interestingly, the dependence on $\Gamma$ is the same as the one obtained by \citet{brunskill2013sample}, but in our case it constitutes a worst-case scenario since $\Gamma$ can be regarded as the \emph{minimum} positive information for model discrimination, while Theorem \ref{th:pac-bound} scales with the maximum one. Moreover, since our sample complexity bound is never worse than the one of \citet{brunskill2013sample} and theirs achieves robustness to negative transfer, PTUM directly inherits this property.
We note that $\Gamma > 0$ since, otherwise, two identical models would exist and one could be safely neglected. The key consequence is that one could set $\epsilon = 0$ and the algorithm would retrieve an optimal policy. However, this requires the models to be perfectly approximated so as to enter the transfer mode. Finally, we point out that the optimal dependence on the discount factor was proved to be $\mathcal{O}(1/(1-\gamma)^3)$~\cite{azar2013minimax}. Here, we get a slightly sub-optimal result since, for simplicity, we naively upper-bounded the variances with their maximum value, but a more involved analysis (e.g., those by~\citet{azar2013minimax,sidford2018near}) should lead to optimal dependence.

\subsection{Discussion}

The sampling procedure of PTUM (Step 4) relies on the availability of a generative model to query informative state-action pairs. We note that the definition of informative state-action pair and all other components of the algorithm are independent on the assumption that a generative model is available. Therefore, one could use ideas similar to PTUM even in the absence of a generative model. For instance, taking inspiration from E$^3$ \citep{kearns2002near}, we could build a surrogate ``exploration" MDP with high rewards in informative state-action pairs and solve this MDP to obtain a policy that autonomously navigates the true environment to collect information. Alternatively, we could use the information measure $\Psi_{s,a}$ as an exploration bonus \citep{jian2019exploration}. We conjecture that the analysis of this kind of approaches would follow quite naturally from the one of PTUM under the standard assumption of finite MDP diameter $D < \infty$ \citep{jaksch2010near}, for which the resulting bound would have an extra linear scaling in $D$ as in prior works \citep{brunskill2013sample}.

PTUM calls an $(\epsilon,\delta)$-PAC algorithm whenever the models are too inaccurate or whenever $n$ queries to the generative model are not sufficient to identify a near-optimal policy. This algorithm can be freely chosen among those available in the literature. For instance, we could choose the MaxQInit algorithm of \citet{abel2018policy} which uses an optimistic value function to initialize the learning process of a PAC-MDP method \citep{strehl2009reinforcement}. In our case, the information about $\theta^*$ collected through the generative model could be used to compute much tighter upper bounds to the optimal value function than those obtained solely from previous tasks, thus significantly reducing the overall sample complexity. Alternatively, we could use the Finite-Model-RL algorithm of \cite{brunskill2013sample} or the Parameter ELimination (PEL) method of \cite{dyagilev2008efficient} by passing the set of survived models $\bar{\Theta}_n$ instead of $\Theta$, so that the number of remaining eliminations is potentially much smaller than $|\Theta|$.


\section{Learning Sequentially-Related Tasks}\label{sec:structure}

We now show how to estimate the MDP models (i.e., the reward and transition probabilities) for each $\theta$ sequentially,  together with the task-transition matrix $T$. The method we propose allows us to obtain confidence bounds over these estimates which can be directly plugged into Algorithm~\ref{alg:policy-transfer} to find an $\epsilon$-optimal policy for each new task.

We start by noticing that our sequential transfer setting can be formulated as a hidden Markov model (HMM) where the target tasks $\{\theta_h^*\}_{h\geq 1}$ are the unobserved variables, or hidden states, which evolve according to $T$, and the samples collected while learning each task are the agent's observations. Formally, consider the $h$-th task $\theta^*_h$ and, with some abuse of notation, let $\hat{q}_h(u | s,a)$ and $\hat{p}_h(s' | s,a)$ respectively denote the empirical reward distribution and transition model estimated after solving $\theta^*_h$ and, without loss of generality, after collecting at least one sample from each state-action pair. In order to simplify the exposition and analysis, we assume, in this section only, that the reward-space is finite with $U$ elements.\footnote{The proposed methods can be applied to continuous reward distributions, e.g., Gaussian, with only minor tweaks.} It is easy to see that $\expec{\hat{q}_h(u | s,a) | \theta^*_h = \theta} = q_\theta(u | s,a)$ and $\expec{\hat{p}_h(s' | s,a) | \theta^*_h = \theta} = p_\theta(s' | s,a)$. Furthermore, the random variables $\hat{q}_h(u | s,a)$ and  $\hat{p}_h(s' | s,a)$ are conditionally independent of all other variables given the target task $\theta^*_h$. Let $o_h \in \R^{d}$, for $d = SA(S+U)$, be a vectorized version of these empirical models (i.e., our observation vector), $o_h = [\mathrm{vec}(\hat{q}); \mathrm{vec}(\hat{p})]$. Let $O \in \R^{d\times k}$ be defined as $O = [\mu_{1}, \mu_2, \dots, \mu_k]$ with $\mu_{j} := \expec{o_h | \theta_h^* = \theta_j}$. Intuitively, $O$ has the (exact) flattened MDP models as its columns and can be regarded as the matrix of conditional-observation means for our HMM. Then, the problem reduces to estimating the matrices $T$ and $O$ while facing sequential tasks, which is a standard HMM learning problem. Spectral methods~\cite{anandkumar2012method, anandkumar2014tensor} have been widely applied for this purpose. Besides their good empirical performance, these approaches typically come with strong finite-sample guarantees, which is a fundamental tool for our study. We also note that spectral methods have already been applied to solve problems related to ours, including transfer in multi-armed bandits~\cite{azar2013sequential} and learning POMDPs~\cite{azizzadenesheli2016reinforcement, guo2016pac}. Here, we apply the tensor decomposition approach of~\citet{anandkumar2014tensor}. As their algorithm is quite involved, we include a brief description in Appendix~\ref{app:rtp}, while in the following we focus on analyzing the estimated HMM parameters. 

Our high-level sequential transfer procedure, which combines PTUM with HMM learning, is presented in Algorithm~\ref{alg:sequential-transfer}. The approach keeps a running estimate of all MDP models $\{\tilde{\M}_\theta^h\}_{\theta\in\Theta}$ together with a bound on their errors $\Delta_h$. These estimates are used to solve each task via PTUM (line 4) and updated by plugging the resulting observations into the spectral method (lines 5-7). Then, the error bounds of the new estimates are computed as explained in Section~\ref{sec:error-bounds-spectral} (line 8). Finally, the estimated task-transition matrix is used to predict which tasks are more likely to occur and the resulting set $\tilde{\Theta}_{h+1}$ is used to run PTUM at the next step.

\begin{algorithm}[t]
\caption{Sequential Transfer} \label{alg:sequential-transfer}
\begin{algorithmic}[1]

\STATE{Initialize set of approximate models $\{\tilde{\M}_\theta^1\}_{\theta\in\Theta}$, model errors $\Delta_1 \leftarrow \infty$, and initial active set $\tilde{\Theta}_1 \leftarrow \Theta$}
\FOR{$h=1,2,\dots$}
\STATE{Receive new task $\theta^*_h \sim T(\cdot | \theta^*_{h-1})$}
\STATE{Run Algorithm \ref{alg:policy-transfer} with $\{\tilde{\M}_\theta^h\}_{\theta\in\tilde{\Theta}_{h}}$ and $\Delta_h$}
\STATE{Obtain estimates $\hat{q}_h(u | s,a)$ and $\hat{p}_h(s' | s,a)$}
\STATE{Estimate $\hat{O}_h, \hat{T}_h$ using Algorithm \ref{alg:hmm-learning} in Appendix \ref{app:rtp}}
\STATE{Update model estimates $\{\tilde{\M}^{h+1}_\theta\}_{\theta\in\Theta}$ from $\hat{O}_h$}
\STATE{Compute model error $\Delta_h$}
\STATE{Predict set of initial models $\tilde{\Theta}_{h+1}$ from $\hat{T}_h$}
\ENDFOR

\end{algorithmic}
\end{algorithm}

\subsection{Error Bounds for the Estimated Models}\label{sec:error-bounds-spectral}

We now derive finite-sample bounds on the estimation error of each MDP model. First, we need the following assumption, due to~\citet{anandkumar2012method}, to ensure that we can recover the HMM parameters from the samples.
\begin{assumption}\label{ass:full-rank}
The matrix $O$ is full column-rank. Furthermore, the stationary distribution of the underlying Markov chain, $\omega \in \mathcal{P}(\Theta)$, is such that $\omega(\theta) > 0$ for all $\theta \in \Theta$.
\end{assumption}
As noted by \citet{anandkumar2012method}, this assumption is milder than assuming minimal positive gaps between the different models (as was done, e.g., by~\citet{brunskill2013sample} and~\citet{liu2016pac}). We now present the main result of this section, which provides deviation inequalities between true and estimated models that hold uniformly over time.
\begin{restatable}{theorem}{rtpmodels}\label{th:rtp-err-models}
Let $\{\tilde{\M}_\theta^h\}_{\theta\in\Theta, h\geq1}$ be the sequence of MDP models estimated by Algorithm \ref{alg:sequential-transfer}, with $\tilde{p}_{h,\theta}(s' | s,a)$ and $\tilde{q}_{h,\theta}(u | s,a)$ the transition and reward distributions, $\tilde{V}_{h,\theta}^*$ the optimal value functions, and $\tilde{\sigma}_{h,\theta}^r(s,a)$, $\tilde{\sigma}_{h,\theta}^p(s,a; \theta')$ the corresponding variances. There exist constants $\rho, \rho_r, \rho_p, \rho_{\sigma_r}, \rho_{\sigma_p}$ such that, for $\delta' \in (0,1)$, if
\begin{align*}
    h > \rho \log\frac{ 2h^2SA(S+U)}{\delta'},
\end{align*}
then, with probability at least $1-\delta'$, the following hold simultaneously for all $h\geq1$, $s\in\mathcal{S}$, $a\in\mathcal{A}$, and $\theta,\theta'\in\Theta$:
\begin{align*}
    |r_\theta(s,a) - \tilde{r}_{h,\theta}(s,a)| &\leq \rho_r \sqrt{\frac{\log c_{h,\delta'}}{h}}, \\
    |(p_\theta(s,a) - \tilde{p}_{h,\theta}(s,a))\tr \tilde{V}_{h,\theta'}^*| &\leq \rho_p \sqrt{\frac{\log c_{h,\delta'}}{h}}, \\ 
        |\sigma_\theta^r(s,a) - \tilde{\sigma}^r_{h,\theta}(s,a)| &\leq \rho_{\sigma_r} \sqrt{\frac{\log c_{h,\delta'}}{h}}, \\
    |\sigma_\theta^r(s,a; \theta') - \tilde{\sigma}^p_{h,\theta}(s,a; \theta')| &\leq \rho_{\sigma_p} \sqrt{\frac{\log c_{h,\delta'}}{h}},
\end{align*}
where $c_{h,\delta'} := \pi^2h^2SA(S+U)/\delta'$.
\end{restatable}
To favor readability, we collapsed all terms of minor relevance to our purpose into the constants $\rho$. These are functions of the given family of tasks (through maximum/minimum eigenvalues of the covariance matrices introduced previously) and of the underlying Markov chain. Full expressions are reported in Appendix~\ref{app:proof4}. These bounds provide us with the desired deviations between the true and estimated models and can be used as input to the PTUM algorithm to learn each task in the sequential setup. It is important to note that, like other applications of spectral methods to RL problems~\cite{azar2013sequential, azizzadenesheli2016reinforcement}, the constants $\rho$ are not known in practice and should be regarded as a parameter of the algorithm. Such a parameter can be interpreted as the confidence level in a standard confidence interval. Finally, regardless of constants, these deviations decay at the classic rate of $\mathcal{O}(1/\sqrt{h})$, so that the agent recovers the perfect models of all tasks in the asymptotic regime. This also  implies that the agent needs to observe $\mathcal{O}(1/\epsilon^2)$ tasks before PTUM enters the transfer mode, which is theoretically tight in the worst-case \cite{feng2019does} but quite conservative in practice.

\subsection{Exploiting the Inter-Task Dynamics}

We now show how the agent can exploit the task-transition matrix $T$ to further reduce the sample complexity. Suppose that $T$ and the identity of the current task $\theta_h^*$ are known. Then, through $T(\cdot | \theta^*_h)$, the agent has prior knowledge about the next task and it could, for instance, discard all models whose probability is low enough without even passing them to the PTUM algorithm. Using this insight, we can design an approach to pre-process the initial set of models by replacing $T$ with its estimate $\hat{T}$.
Let $\hat{T}_h$ be the estimated task-transition matrix and $\bar{\Theta}_h$ be the set of active models returned by Algorithm~\ref{alg:policy-transfer}. Then, by design of the algorithm, $\bar{\Theta}_h$ contains the target task $\theta^*_h$ with sufficient probability and thus we can predict the probability that $\theta^*_{h+1}$ is equal to some $\theta$ as $\sum_{\theta'\in\bar{\Theta}_h} \hat{T}_h(\theta,\theta')$. Our idea is to check whether this probability, plus some confidence value, is lower than a certain threshold $\eta$. In such a case, we discard $\theta$ from the initial set of models $\tilde{\Theta}_{h+1}$ to be used at the next step. The following theorem ensures that, for a well-chosen threshold value, the target model is never discarded from the initial set at any time with high probability.
\begin{restatable}{theorem}{preelim}\label{th:pre-elim}
Let $\delta' \in (0,1)$ and $\delta \leq \frac{\delta'}{3m^2}$ be the confidence value for Algorithm \ref{alg:policy-transfer}. Suppose that, before each task $h$, a model $\theta$ is eliminated from the initial active set if:
\begin{align*}
    \sum_{\theta'\in\bar{\Theta}_h} \hat{T}_h(\theta,\theta') + \delta k + \rho_Tk \sqrt{\frac{\log (9kdm^2/\delta')}{h}} \leq \eta.
\end{align*}
Then, for $\eta = \frac{\delta'}{3km^2}$, with probability at least $1-\delta'$, at any time the true model is never eliminated from the initial set.
\end{restatable}
As before, we collapsed minor terms into the constant $\rho_T$. When $T$ is sparse (e.g., only a small number of successor tasks is possible), this technique could significantly reduce the sample complexity to identify a near-optimal policy as many models are discarded even before starting PTUM.

\section{Related Works}

\begin{figure*}[t]
\centering
\begin{minipage}[t]{0.49\linewidth}
\centering
\includegraphics[height=3.3cm]{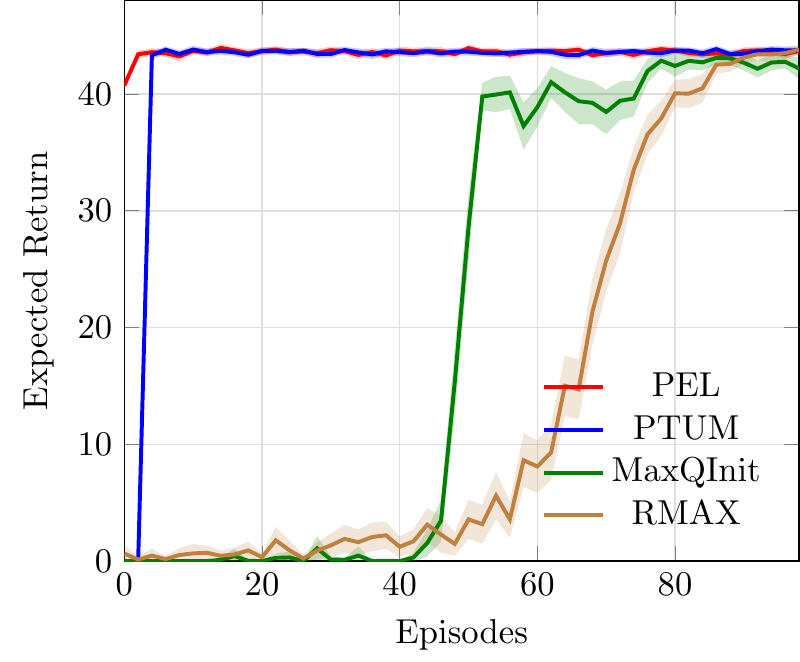}
\includegraphics[height=3.3cm]{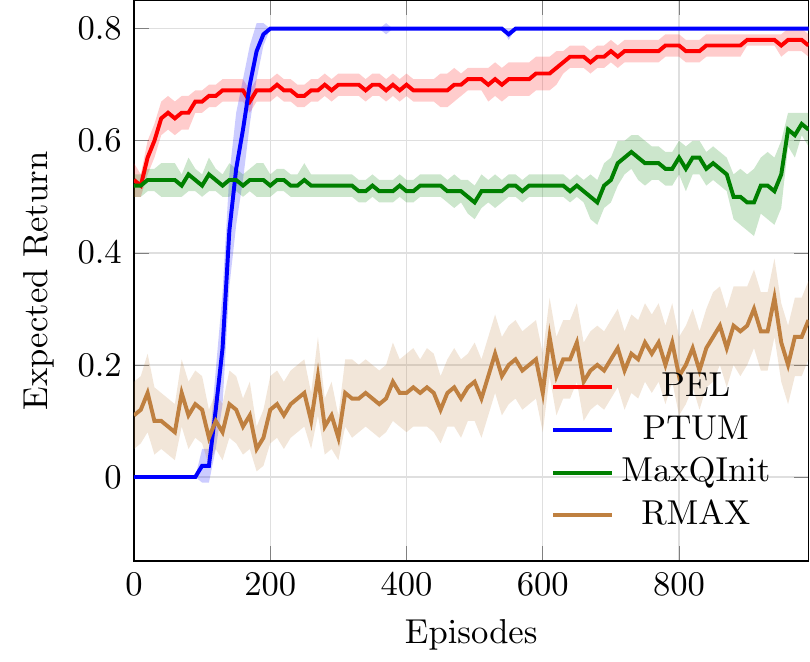}
\caption{Policy transfer from known models. \textit{(left)} Optimism gains information.  \textit{(right)} High-reward states are poorly informative.}
    \label{fig:known-models}
\end{minipage}\hfill
\begin{minipage}[t]{0.49\linewidth}
\centering
\includegraphics[height=3.3cm]{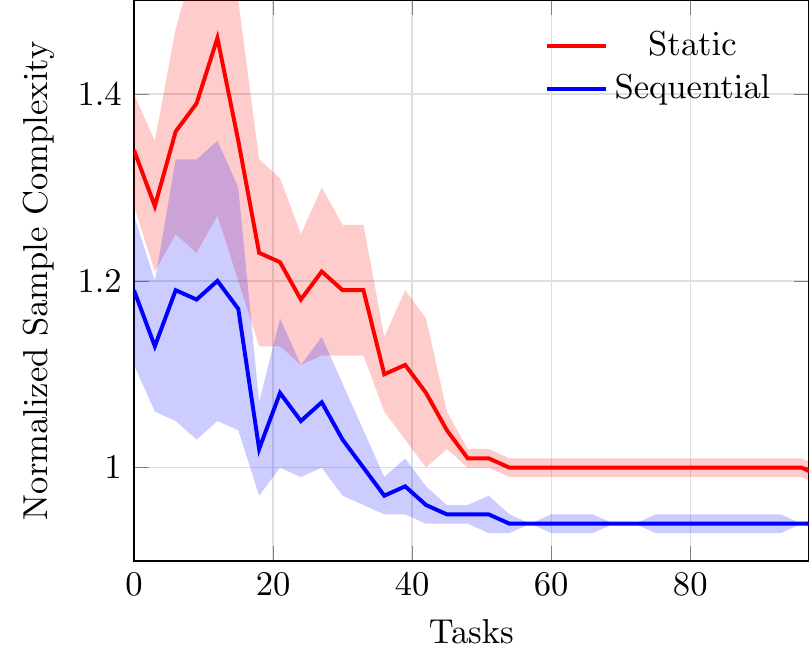} \hfill
\includegraphics[height=3.3cm]{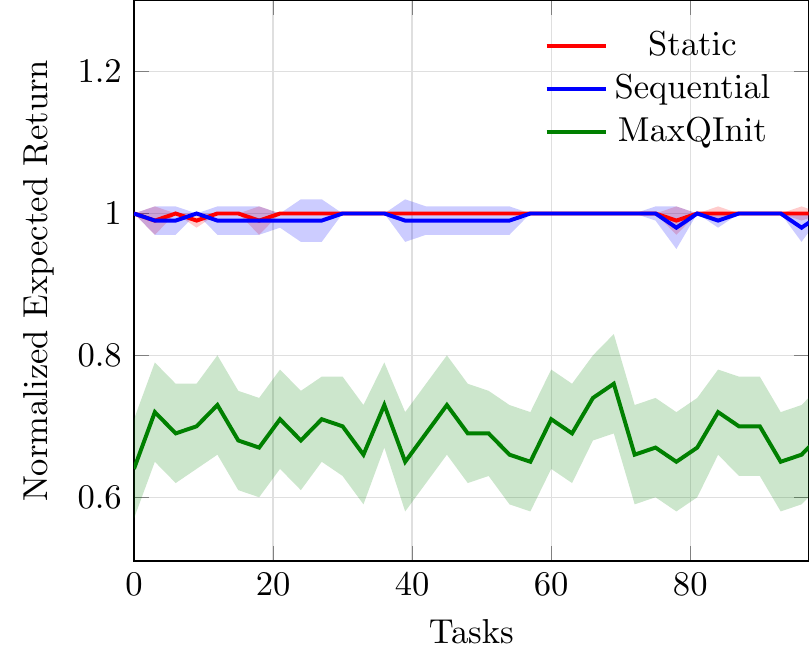}
\caption{Sequential transfer experiment.  \textit{(left)} The normalized sample complexity.  \textit{(right)} Performance of the computed policies.} 
    \label{fig:sequential}
\end{minipage}
\end{figure*}

In recent years, there has been a growing interest in understanding transfer in RL from a theoretical viewpoint. The work by~\citet{brunskill2013sample} and its follow-up by~\citet{liu2016pac} are perhaps the most related to ours. The authors consider a multi-task scenario in which the agent faces finitely many tasks drawn from a fixed distribution. Although their methods are designed for task identification, they do not actively seek information as our algorithm does and the resulting guarantees are therefore weaker. \citet{azar2013sequential} present similar ideas in the context of multi-armed bandits. Like our work, they employ spectral methods to learn tasks sequentially, without however exploiting any temporal correlations. \citet{dyagilev2008efficient} propose a method for task identification based on sequential hypothesis testing. Their algorithm explores the environment by running optimistic policies and, like ours, aims at eliminating models from a confidence set. \citet{mann2012directed} and~\citet{abel2018policy} study how to to achieve jumpstart. The idea is to optimistically initialize the value function of the target task and then run any PAC algorithm. The resulting methods provably achieve positive transfer. \citet{mahmud2013clustering} cluster solved tasks in a few groups to ease policy transfer, which relates to our spectral learning of MDP models. Our setup is also related to other RL settings, including contextual MDPs~\cite{hallak2015contextual, jiang2017contextual, modi2018markov, sun2018model}, where the decision process depends on latent contexts, and hidden-parameter MDPs~\cite{velez2013hidden, killian2017robust}, which follow a similar idea. Finally, we notice that the PAC-MDP framework has been widely employed to study RL algorithms~\cite{strehl2009reinforcement}. Many of the above-mentioned works take inspiration from classic PAC methodologies, including the well-known E3~\cite{kearns2002near}, R-MAX~\cite{brafman2002r}, and Delayed Q-learning~\cite{strehl2006pac}, and so does ours.

\section{Experiments}

The goal of our experiments is twofold. First, we analyze the performance of PTUM when the task models are known, focusing on the comparison between identification and jumpstart strategies. Then, we consider the sequential setting and show that our approach progressively transfers knowledge and exploits the temporal task correlations to improve over "static" transfer methods.  Due to space constraints, here we present only the high-level setup of each experiment and discuss the results. We refer the reader to Appendix~\ref{app:exp} for all the details and further results. All our plots report averages of $100$ runs with $99\%$ Student's t confidence intervals.

\paragraph*{Identification vs Jumpstart}

For this experiment, we adopt a standard $12\times 12$ grid-world divided into two parts by a vertical wall  (i.e., the $4$-rooms domain of \citet{sutton1999between} with only two rooms). The agent starts in the left room and must reach a goal state in the right one, with the two rooms connected by a door. We consider $12$ different tasks, whose models are known to the agent, with different goal and door locations. We compare PTUM with three baselines: RMAX \cite{brafman2002r}, which does not perform any transfer, RMAX with MaxQInit \cite{abel2018policy}, in which the Q-function is initialized optimistically to achieve a jumpstart, and PEL \cite{dyagilev2008efficient}, which performs model identification. Not to give advantage to PTUM, as the other algorithms run episodes online with no generative model, in our plots we count each sample taken by PTUM as one episode and assign it the minimum reward value. Once PTUM returns a policy, we report its online performance. The results in Figure \ref{fig:known-models}\textit{(left)} show that the two identification methods find an optimal policy in a very small number of episodes. The jumpstart of MaxQInit is delayed since the optimistic Q-function directs the agent towards the wall, but finding the door takes many samples. PEL, which is also optimistic, instantly identifies the solution since attempts to cross the wall in locations without a door immediately discriminate between tasks. This is in fact an example where the optimism principle leads to both rewards and information. To understand why this is not always the case, we run the same experiment by removing the wall and adding multiple goal locations with different reward values. Some goals have uniformly high values over all tasks (thus they provide small information) while others are more variable (thus with higher information). In Figure~\ref{fig:known-models}\textit{(right)}, we see that both PEL and MaxQInit achieve a high jumpstart, as they immediately go towards one of the goals, but converge much more slowly than PTUM, which focuses on the informative ones.

\paragraph*{Sequential Transfer}

We consider a variant of the objectworld domain by \citet{levine2011nonlinear}. Here, we have a $5\times 5$ grid where each cell can contain one of $11$ possible items with different values (or no items at all). There are $8$ possible tasks, each of which randomizes the items with probabilities inversely proportional to their value. In order to show the benefits of exploiting temporal correlations, we design a sparse task-transition matrix in which each task has only a few possible successors. We run our sequential transfer method (Algorithm \ref{alg:sequential-transfer}) and compare it with a variant that ignores the inter-task dynamics and only performs "static" transfer (i.e., only the estimated MDP models are used). We note that model-learning in this variant reduces to the spectral method proposed by \citet{azar2013sequential} for bandits. For the initial phase of the learning process, where the transfer condition is not satisfied, we solve each task by sampling state-action pairs uniformly as in \citet{azar2013minimax}. Figure \ref{fig:sequential}\textit{(left)} shows the normalized number of samples required by PTUM to solve each task starting from the first step in which the algorithm enters the transfer mode\footnote{The sample complexity to solve each task is divided by the one of the oracle version of PTUM with known models}. We can appreciate that the sample complexities of both variants decay as the number of tasks grows, meaning that the learned models become more and more accurate. As expected, the sequential version outperforms the static one by exploiting the temporal task correlations. It is important to note that the normalized sample complexity of our algorithm goes below $1$ (i.e., below the oracle). This is a key point as the sample complexity of oracle PTUM is computed with the set of all models as input, while the sequential approach filters this set based on the estimated task-transition matrix. In Figure \ref{fig:sequential}\textit{(right)}, we confirm that both variants of our approach return near-optimal policies at all steps (all are $\epsilon$-optimal for the chosen value of $\epsilon$). We also report the performance of the policies computed by MaxQInit when the same maximum-sample budget as our approach is given (to be fair, we overestimated it as $5$ to $10$ times larger than what oracle PTUM needs). Although for MaxQInit we used the oracle optimistic initialization (with the true Q-functions), the algorithm is not able to obtain near-optimal performance using the given budget, while PTUM achieves it despite the estimated models.

\section{Conclusion}

We studied two important questions in lifelong RL. First, we addressed the problem of quickly identifying a near-optimal policy to be transferred among the solutions of previously-solved tasks. The proposed approach sheds light on how to design strategies that actively seek information to reduce identification time, and our sample complexity bounds confirm a significant positive transfer. Second, we showed that learning sequentially-related tasks is not significantly harder than the common case where they are statically sampled from a fixed distribution. By combining these two building blocks, we obtained an algorithm that sequentially learns the task structure, exploits temporal correlations, and quickly identifies the solution of each new task. We confirmed these properties in simple simulated domains, showing the benefits of identification over jumpstart strategies and of exploiting temporal task relations.

Our work opens up several interesting directions. Our ideas for task identification could be extended to the online setting, where a generative model is not available. Here, we would probably need a more principled trade-off between information and rewards, which can be modeled by different objectives (such as regret). 
Furthermore, it would be interesting to study the sequential transfer setting in large continuous MDPs. Here we could take inspiration from recent papers on meta RL that learn latent task representations \citep{humplik2019meta,lan2019meta,rakelly2019efficient,zintgraf2019varibad}  and meta-train policies that seek information for quickly identifying the task parameters.

\bibliography{biblio}
\bibliographystyle{icml2020}

\newpage
\appendix
\onecolumn

\section{Notation} \label{app:notation}

\begin{table*}[h]
\centering
\begin{tabular}{@{}lll@{}} 
\toprule
Symbol & Expression &Meaning \\
\cmidrule{1-3}
$\Theta$ & - & Set of parameters\\
$k$ & $|\Theta|$ & Number of possible parameters/tasks\\
$n$ & - & Sample budget for querying the generative model (Section \ref{sec:transfer})\\
$m$ & - & Maximum number of tasks in the sequential setting (Section \ref{sec:structure})\\
$\mathcal{S}$ & - & Set of $S$ states\\
$\mathcal{A}$ & - & Set of $A$ actions\\
$\mathcal{U}$ & - & Set of reward values (finite with cardinality $U$ for Section \ref{sec:structure} only)\\
$T$ & - & Task-transition matrix\\
$p_\theta(s' | s,a)$ & - & Transition probabilities of MDP $\mathcal{M}_\theta$\\
$q_\theta(u | s,a)$ & - & Reward distribution of MDP $\mathcal{M}_\theta$\\
$r_\theta(s,a)$ & - & Mean reward of MDP $\mathcal{M}_\theta$\\
$\gamma$ & - & Discount factor\\
$V_\theta^\pi(s)$ & $\E_\theta^\pi\left[\sum_{t=0}^\infty \gamma^t U_t | S_0 = s\right]$ & Value function of policy $\pi$ in MDP $\mathcal{M}_\theta$\\
$V_\theta^*(s)$ & $\max_\pi V_\theta^\pi(s)$ & Optimal value function for MDP $\mathcal{M}_\theta$\\
$\sigma_\theta^r(s,a)^2$ & $\Var_{q_\theta(\cdot | s,a)}[U]$ & Reward variance in task $\M_\theta$\\
 $\sigma_\theta^p(s,a;\theta')^2$ & $\Var_{p_\theta(\cdot | s,a)}[V_{\theta'}^*(S')]$ & Transition/value-function variance in task $\M_\theta$\\
$\devr{s,a}{\theta,\theta'}$ & $|{r}_\theta(s,a) - {r}_{\theta'}(s,a)|$ & Reward-gaps between tasks $\theta$ and $\theta'$ \\
$\devp{s,a}{\theta,\theta'}{}$ & $|({p}_\theta(s,a) - {p}_{\theta'}(s,a))\tr V^*_\theta|$ & Transition-gaps between tasks $\theta$ and $\theta'$ \\
$\Delta$ & - &  Estimation error of the approximate models (Assumption \ref{ass:model-err})\\
$\epsilon, \delta$ & - & Accuracy and confidence level for Algorithm \ref{alg:policy-transfer}\\
$\hat{r}_t, \hat{p}_t, \hat{\sigma}_t^r, \hat{\sigma}_t^p$ & See Algorithm \ref{alg:policy-transfer} & Empirical models after $t$ steps \\
$C_{t,\delta}^x(s,a)$ & See Lemma \ref{lemma:conf} & Bernstein confidence intervals for $x \in \{r,p,\sigma_r,\sigma_p\}$\\
$\bar{\Theta}_t$ & See Algorithm \ref{alg:policy-transfer} & Confidence set a t time $t$\\
$\Psi_{s,a}^r(\theta,\theta')$ & See Definition \ref{def:psi} & Reward information for discriminating $\theta,\theta'$\\
$\Psi_{s,a}^p(\theta,\theta')$ & See Definition \ref{def:psi} & Transition information for discriminating $\theta,\theta'$\\
$\Psi_t(s,a)$ & $\max \left\{  \Psi_t^r(s,a), \Psi_t^p(s,a)\right\}$ & Index of $s,a$ at time $t$ (see Algorithm \ref{alg:policy-transfer})\\
$O$ & - & Mean-observation matrix containing the flattened MDP models\\
$\hat{O}_h, \hat{T}_h$ & - & Estimated observation and task-transition matrices after $h$ tasks\\
$\tilde{r}_{h,\theta},\tilde{p}_{h,\theta},\tilde{\sigma}^r_{h,\theta},\tilde{\sigma}^p_{h,\theta}$ & - & Estimated models after $h$ tasks\\
$\tilde{\Theta}_h$ & - & Initial set of models for running PTUM on the $h$-th task\\
$\rho_x$ or $ \rho_x(\Theta,T)$ & See Appendix \ref{app:proof4} & Constants in the analysis of the spectral learning algorithm\\
\bottomrule
\end{tabular}
\caption{The notation adopted in this paper.}
\label{tab:notation}
\end{table*}

\newpage

\section{Analysis of the PTUM Algorithm}\label{app:proof3}

\subsection{Definitions and Assumptions}

The analysis is carried out under the following two assumptions.

\begin{assumption}\label{ass:ptum1}
Algorithm \ref{alg:policy-transfer} always enters the transfer mode. That is, the model uncertainty is such that $\Delta < \frac{\epsilon(1-\gamma)}{4(1+\gamma)}$.
\end{assumption}

\begin{assumption}\label{ass:ptum2}
The sample budget $n$ is large enough to allow Algorithm \ref{alg:policy-transfer} to identify an $\epsilon$-optimal policy.
\end{assumption}

These two assumptions allow us to analyze only the core part of PTUM (i.e., the transfer mode), thus excluding trivial cases in which the chosen $(\epsilon,\delta)$-PAC algorithm is called.  In fact, if Assumption \ref{ass:ptum1} does not hold, the sample complexity for computing an $\epsilon$-optimal policy is equivalent to the one of the chosen algorithm. Similarly, if Assumption \ref{ass:ptum2} does not hold, the sample complexity is $n$ (the samples collected by the generative model) plus the sample complexity of the chosen algorithm.

We define the event $E:=\{\forall t = 1,...,n: \theta^* \in \bar{\Theta}_t \}$ under which the true model is never eliminated from the active model set. This event will be used extensively throughout the whole analysis.

\subsection{Concentration Inequalities}

\begin{lemma}[Bernstein's inequality \citep{boucheron2003concentration}]\label{lemma:bernstein}
Let $X$ be a random variable such that $|X| \leq c$ almost surely, $X_1,\dots,X_n$ $n$ i.i.d. samples of $X$, and $\delta > 0$. Then, with probability at least $1-\delta$,
\begin{align*}
    \left|\expec{X} - \frac{1}{n}\sum_{i=1}^n X_i\right| \leq \sqrt{\frac{2\Var[X]\log\frac{2}{\delta}}{n}} + \frac{c\log\frac{2}{\delta}}{3n}.
\end{align*}
\end{lemma}

\begin{lemma}[Empirical Bernstein's inequality \citep{maurer2009empirical}]\label{lemma:emp-bernstein}
Let $X$ be a random variable such that $|X| \leq c$ almost surely, $X_1,\dots,X_n$ $n$ i.i.d. samples of $X$, and $\delta > 0$. Then, with probability at least $1-\delta$,
\begin{align*}
    \left|\expec{X} - \frac{1}{n}\sum_{i=1}^n X_i\right| \leq \sqrt{\frac{2\mathrm{\hat{\mathbb{V}}ar}[X]\log\frac{4}{\delta}}{n}} + \frac{7c\log\frac{4}{\delta}}{3(n-1)},
\end{align*}
where $\mathrm{\hat{\mathbb{V}}ar}[X]$ denotes the empirical variance of $X$ using $n$ samples.
\end{lemma}

\subsection{Lemmas}

We begin by showing that the true model is never eliminated from the confidence sets of Algorithm \ref{alg:policy-transfer} with high probability.

\begin{lemma}[Valid confidence sets]\label{lemma:conf}
Let $\delta > 0$ and, for $N_{t}(s,a) > 1$,
\begin{align*}
	C^{r}_{t, \delta}(s,a) = \sqrt{\frac{2\hat{\sigma}^{r}_{t}(s,a)^2\log{\frac{8SAn(|\Theta|+1)}{\delta}}}{N_t(s,a)}} + \frac{7\log{\frac{8SAn(|\Theta|+1)}{\delta}}}{3(N_t(s,a)-1)} + \Delta_{\mathrm{max}}^r,
\end{align*}
\begin{align*}
	C^{p}_{t,\delta}(s,a; \theta') = \sqrt{\frac{2\hat{\sigma}^{p}_{t}(s,a;\theta')^2\log{\frac{8SAn(|\Theta|+1)}{\delta}}}{N_t(s,a)}} + \frac{7\log{\frac{8SAn(|\Theta|+1)}{\delta}}}{3(N_t(s,a) -1)(1-\gamma)} + \Delta_{\mathrm{max}}^p,
\end{align*}
\begin{align*}
	C^{\sigma^r}_{t,\delta}(s,a) = \sqrt{\frac{2\log{\frac{4SAn(|\Theta|+1)}{\delta}}}{N_t(s,a)-1}} + \Delta_{\mathrm{max}}^{\sigma_r},
\end{align*}
\begin{align*}
	C^{\sigma^p}_{t,\delta}(s,a) = \frac{1}{1-\gamma}\sqrt{\frac{2\log{\frac{4SAn(|\Theta|+1)}{\delta}}}{N_t(s,a)-1}} + \Delta_{\mathrm{max}}^{\sigma_p}.
\end{align*}
Set these confidence intervals to infinity if $N_t(x,a) \leq 1$. Then, the event $E:=\{\forall t = 1,...,n: \theta^* \in \bar{\Theta}_t \}$ holds with probability at least $1-\delta$.
\end{lemma}
\begin{proof}
Take any step $t\geq 1$, any state-action pair $(s,a) \in \mathcal{S}\times\mathcal{A}$, any model $\theta'\in\Theta$, and let $\delta'>0$. We need to show that the conditions of \eqref{eq:single-conf-r}-\eqref{eq:single-conf-sp} hold. First notice that these conditions trivially hold if $N_t(s,a)\leq 1$. Thus, suppose that $N_t(s,a) > 1$ so that the confidence intervals are well-defined. Using the triangle inequality we have that
\begin{align*}
	| \hat{r}_t(s,a) - \tilde{r}_{\theta^*}(s,a) | \leq |\hat{r}_t(s,a) - r_{\theta^*}(s,a)| +  \Delta_{\mathrm{max}}^r
\end{align*}
\begin{align*}
	| (\tilde{p}_{\theta^*}(s,a) - \hat{p}_{t}(s,a))\tr \tilde{V}^{*}_{\theta'} | \leq | (p_{\theta^*}(s,a) - \hat{p}_{t}(s,a))\tr \tilde{V}^{*}_{\theta'} | + \Delta_{\mathrm{max}}^p
\end{align*}
\begin{align*}
	|\hat{\sigma}^{r}_{t}(s,a) - \tilde{\sigma}^{r}_{\theta^*}(s,a) | \leq |\hat{\sigma}^{r}_{t}(s,a) - \sigma^{r}_{\theta^*}(s,a) | + \Delta_{\mathrm{max}}^{\sigma_r},
\end{align*}
\begin{align*}
	|\hat{\sigma}^{p}_{t}(s,a;\theta') - \hat{\sigma}^{p}_{\theta^*}(s,a;\theta') | \leq |\hat{\sigma}^{p}_{t}(s,a;\theta') - \sigma^{p}_{\theta^*}(s,a;\theta') | + \Delta_{\mathrm{max}}^{\sigma_p}.
\end{align*}
Using Lemma \ref{lemma:emp-bernstein}, we have that, with probability at least $1-\delta'$,
\begin{align*}
    |\hat{r}_t(s,a) - r_{\theta^*}(s,a)| \leq \sqrt{\frac{2\hat{\sigma}_t^r(s,a)^2\log\frac{4}{\delta'}}{N_t(s,a)}}+\frac{7\log\frac{4}{\delta'}}{3(N_{t}(s,a)-1)}.
\end{align*}
Similarly, for any $\theta' \in \Theta$, we have that, with probability at least $1-\delta'$,
\begin{align*}
    |(p_{\theta^*}(s,a)-\hat{p}_{t}(s,a))\tr \tilde{V}_{\theta'}^*| \leq \sqrt{\frac{2\hat{\sigma}_t^p(s,a;\tilde{V}_{\theta'}^*)^2\log\frac{4}{\delta'}}{N_t(s,a)}}+\frac{7\log\frac{4}{\delta'}}{3(N_{t}(s,a)-1)(1-\gamma)}.
\end{align*}
From Theorem 10 of \cite{maurer2009empirical},
\begin{align*}
    |\hat{\sigma}_t^r(s,a) - \sigma_{\theta^*}^r(s,a)| \leq \sqrt{\frac{2\log\frac{2}{\delta'}}{N_{t}(s,a)-1}}
\end{align*}
and
\begin{align*}
    |\hat{\sigma}_t^p(s,a;\theta') - \sigma_{\theta^*}^p(s,a;\theta')| \leq \frac{1}{1-\gamma}\sqrt{\frac{2\log\frac{2}{\delta'}}{N_{t}(s,a)-1}}
\end{align*}
hold with probability at least $1-\delta'$, respectively.

Taking union bounds over all state action pairs and over the maximum number of samples $n$, these four inequalities hold at the same time with probability at least $1 - 2SA(|\Theta|+1)n\delta'$. The result follows after setting $\delta = 2SAn(|\Theta|+1)\delta'$ and rearranging.

\end{proof}

Next we bound the number of samples required from some state-action pair in order to eliminate a model from the confidence set.

\begin{lemma}[Model elimination]\label{lemma:model-elim}
Let $\theta \in \Theta$, $(s,a) \in \mathcal{S} \times \mathcal{A}$, $\Delta := \max\{\Delta_{\mathrm{max}}^r, \Delta_{\mathrm{max}}^p, \Delta_{\mathrm{max}}^{\sigma_r}, \Delta_{\mathrm{max}}^{\sigma_p} \}$, and define
\begin{align*}
\bar{n}_{\theta}^{r}(s,a) := \min_{\theta''\in\bar{\Theta}_t}\max \left \{ \frac{ \tilde{\sigma}^{r}_{\theta''}(s,a)^2}{[\adevr{s,a}{\theta^*,\theta} - 4\Delta]_+^2}, \frac{1}{[\adevr{s,a}{\theta^*,\theta} - 4\Delta]_+} \right \} ,
\end{align*}
\begin{align*}
\bar{n}_{\theta}^p(s,a) := \min_{\theta'\in\Theta,\theta''\in\bar{\Theta}_t}\max \left \{ \frac{ \tilde{\sigma}^{p}_{\theta''}(s,a;\theta')^2}{[\adevp{s,a}{\theta^*,\theta}{\tilde{V}^*_{\theta'}} - 4\Delta]_+^2}, \frac{1/(1-\gamma)}{[\adevp{s,a}{\theta^*,\theta}{\tilde{V}^*_{\theta'}} - 4\Delta]_+} \right \}.
\end{align*}
Then, under event E, if $N_{t}(s,a) > \bar{n}_{\theta}(s,a) := 32 \log{\frac{8SAn(1+|\Theta|)}{\delta}} \min \{ \bar{n}_{\theta}^{r}(s,a), \bar{n}_{\theta}^{p}(s,a) \}$, we have that $\theta \notin \bar{\Theta}_{t}$.
\end{lemma}

\begin{proof}
We split the proof into two parts, dealing with rewards and transitions separately. We then combine these results to obtain the final statement.
\paragraph{Elimination by rewards}
Assuming ${\theta} \in \bar{\Theta}_{t}$, we must have, for all state-action pairs and all $\theta''\in\bar{\Theta}_t$,
\begin{align*}
	\adevr{s,a}{\theta^*,\theta} &\leq |\tilde{r}_{\theta}(s,a) - \hat{r}_{t}(s,a)| + |\tilde{r}_{\theta^{*}}(s,a) - \hat{r}_{t}(s,a)| \leq 2 C^{r}_{t,\delta}(s,a) \\  &\le 2 \sqrt{\frac{ 2\hat{\sigma}^{r}_{t}(s,a)^2 \log{\frac{8SAn(|\Theta| + 1)}{\delta}}}{N_{t}(s,a)}} + \frac{14 \log{\frac{8SAn(|\Theta|+1)}{\delta}}}{3(N_{t}(s,a)-1)} + 2\Delta_{\mathrm{max}}^r \\ &\leq 2\sqrt{\frac{2\tilde{\sigma}_{\theta''}^r(s,a)^2\log\frac{8SAn(|\Theta|+1)}{\delta}}{N_t(s,a)}}+\frac{4\log\frac{8SAn(|\Theta|+1)}{\delta}}{(N_{t}(s,a)-1)}+\frac{14\log\frac{8SAn(|\Theta|+1)}{\delta}}{3(N_{t}(s,a)-1)} + 2\Delta_{\mathrm{max}}^r + 2\Delta_{\mathrm{max}}^{\sigma_r} \\ &\le 2 \sqrt{\frac{ 2\tilde{\sigma}^{r}_{\theta''}(s,a)^2 \log{\frac{8SAn(|\Theta| + 1)}{\delta}}}{N_{t}(s,a)}} + \frac{26 \log{\frac{8SAn(|\Theta|+1)}{\delta}}}{3(N_{t}(s,a)-1)} + 4\Delta
\end{align*}
where we applied the triangle inequality and used Lemma \ref{lemma:conf} to upper bound the empirical variance by the variance of a model $\theta''$ in the confidence set. We note that, if the model uncertainty is too high and the denominators of $\bar{n}_\theta(s,a)$ are zero, it is not possible to eliminate $\theta$ from the rewards of this state-action pair. If this is not the case, for $N_{t}(s,a) \ge \bar{n}_{\theta}(s,a)$ we have that
\begin{align*}
	2 \sqrt{\frac{ 2\tilde{\sigma}^{r}_{\theta''}(s,a)^2 \log{\frac{8SAn(|\Theta| + 1)}{\delta}}}{N_{t}(s,a)}} < \frac{	\adevr{s,a}{\theta^*,\theta} - 4\Delta}{2}
\end{align*}
and
\begin{align*}
\frac{26 \log{\frac{8SAn(|\Theta|+1)}{\delta}}}{3(N_{t}(s,a)-1)} < \frac{	\adevr{s,a}{\theta^*,\theta} - 4\Delta}{2}.
\end{align*}
Plugging these two inequalities in the first upper bound leads to the contradiction $\adevr{s,a}{\theta^*,\theta} < \adevr{s,a}{\theta^*,\theta}$, hence it must be that $\theta \notin \bar{\Theta}_{t}$.

\paragraph{Elimination by transition}
The proof proceeds analogously to the previous case. Let $\theta'\in\Theta$ and $\theta''\in\bar{\Theta}_t$, then
\begin{align*}
	\adevp{s,a}{\theta^*,\theta}{\tilde{V}^*_{\theta'}} &\le |(\hat{p}_{t}(s,a) - \tilde{p}_{\theta^{*}}(s,a))\tr \tilde{V}^{*}_{\theta'}| + |(\hat{p}_{t}(s,a) - \tilde{p}_{\theta}(s,a))\tr \tilde{V}^{*}_{\theta'}|  \le 2 C^{p}_{t,\delta}(s,a; \theta') \\  &\le 2 \sqrt{\frac{ 2\hat{\sigma}^{p}_{t}(s,a;\theta') \log{\frac{8SAn(|\Theta| + 1)}{\delta}}}{N_{t}(s,a)}} + \frac{14 \log{\frac{8SAn(|\Theta|+1)}{\delta}}}{3(N_{t}(s,a)-1)(1-\gamma)} + 2\Delta_{\mathrm{max}}^p \\ &\le 2 \sqrt{\frac{ 2\tilde{\sigma}^{p}_{\theta''}(s,a;\theta') \log{\frac{8SAn(|\Theta| + 1)}{\delta}}}{N_{t}(s,a)}} + \frac{26 \log{\frac{8SAn(|\Theta|+1)}{\delta}}}{3(N_{t}(s,a)-1)(1-\gamma)} + 4\Delta,
\end{align*}
where once again we applied the triangle inequality and Lemma \ref{lemma:conf} to upper bound the empirical variance. Hence, applying the same reasoning as before we obtain a contradiction, which in turns implies that $\theta \notin \bar{\Theta}_{t}$.

We finally note that if $\bar{n}_\theta(s,a) = +\infty$, i.e., the approximate models are too inaccurate, it is not possible to eliminate $\theta$ using this state-action pair.
\end{proof}

The following is a known result which bounds the deviation in value function between different MDPs.

\begin{lemma}[Simulation lemma]\label{lemma:simulation}
Let $\theta,\theta' \in \Theta$, $s\in\mathcal{S}$, and $\pi$ be any policy. Denote by $\nu_\theta^\pi(s',a';s)$ the discounted state-action visitation frequencies \cite{sutton2000policy} of $\pi$ in MDP $\theta$ starting from $s$. Then, the following hold:
\begin{align*}
    |V_{\theta}^\pi(s) - V_{\theta'}^{\pi}(s)| \leq \sum_{s',a'} \nu_{\theta'}^{\pi}(s',a';s)\left[ |r_\theta(s',a') - r_{\theta'}(s',a')| + \gamma |(p_{\theta}(s',a')-p_{\theta'}(s',a'))\tr V_{\theta}^\pi|\right].
\end{align*}
\begin{align*}
    |V_{\theta'}^*(s) - V_{\theta}^*(s)| \leq \max_{\pi \in \{\pi^*_\theta, \pi^*_{\theta'}\}} \sum_{s',a'} \nu_{\theta'}^{\pi}(s',a';s)\left[ |r_\theta(s',a') - r_{\theta'}(s',a')| + \gamma |(p_{\theta}(s',a')-p_{\theta'}(s',a'))\tr V_{\theta}^*|\right].
\end{align*}
\end{lemma}
\begin{proof}
See, e.g., Lemma 3 of \citep{zanette2019almost} for the first inequality and Lemma 2 of \citep{azar2013minimax} or Lemma 2 of \citep{zanette2019almost} for the second one.\footnote{Note that, although the inequalities of, e.g., \citet{azar2013minimax} and \citet{zanette2019almost} relate the value functions of a fixed MDP with those of its empirical counterpart, they actually hold for any two MDPs.}
\end{proof}

\begin{corollary}[Value-function error decomposition]\label{cor:v-error-decomposition}
Let $\theta,\theta' \in \Theta$ and $s\in\mathcal{S}$. Then,
\begin{align*}
    |V_{\theta'}^*(s) - V_{\theta'}^{\pi_\theta^*}(s)| \leq 2\max_{\pi \in \{\pi^*_\theta, \pi^*_{\theta'}\}} \sum_{s',a'} \nu_{\theta'}^{\pi}(s',a';s)\left[ |r_\theta(s',a') - r_{\theta'}(s',a')| + \gamma |(p_{\theta}(s',a')-p_{\theta'}(s',a'))^TV_{\theta}^*|\right].
\end{align*}
\end{corollary}

\begin{proof}
Using the triangle inequality,
\begin{align*}
    |V_{\theta'}^*(s) - V_{\theta'}^{\pi_\theta^*}(s)| \leq \underbrace{|V_{\theta'}^*(s) - V_{\theta}^*(s)|}_{(a)} + \underbrace{|V_{\theta}^*(s) - V_{\theta'}^{\pi_\theta^*}(s)|}_{(b)}.
\end{align*}
We can bound (a) using the second inequality in Lemma \ref{lemma:simulation} as
\begin{align*}
    |V_{\theta'}^*(s) - V_{\theta}^*(s)| \leq \max_{\pi \in \{\pi^*_\theta, \pi^*_{\theta'}\}} \sum_{s',a'} \nu_{\theta'}^{\pi}(s',a';s)\left[ |r_\theta(s',a') - r_{\theta'}(s',a')| + \gamma |(p_{\theta}(s',a')-p_{\theta'}(s',a'))^TV_{\theta}^*|\right].
\end{align*}
Similarly, we can use the first inequality in Lemma \ref{lemma:simulation} to bound (b) by noticing that $V^*_\theta = V^{\pi^*_\theta}_\theta$. We have
\begin{align*}
    |V_{\theta}^*(s) - V_{\theta'}^{\pi_\theta^*}(s)| &\leq \sum_{s',a'} \nu_{\theta'}^{\pi_\theta^*}(s',a';s)\left[ |r_\theta(s',a') - r_{\theta'}(s',a')| + \gamma |(p_{\theta}(s',a')-p_{\theta'}(s',a'))^TV_{\theta}^*|\right]] \\ &\leq \max_{\pi \in \{\pi^*_\theta, \pi^*_{\theta'}\}} \sum_{s',a'} \nu_{\theta'}^{\pi}(s',a';s)\left[ |r_\theta(s',a') - r_{\theta'}(s',a')| + \gamma |(p_{\theta}(s',a')-p_{\theta'}(s',a'))^TV_{\theta}^*|\right].
\end{align*}
Combining the two displays above concludes the proof.
\end{proof}

The following lemma ensures that, if the algorithm did not stop at a certain time $t$, certain models belong to the confidence set.

\begin{lemma}[Stopping condition]\label{lemma:theta-eps}
Let $\tau$ be the random stopping time of Algorithm \ref{alg:policy-transfer} and
\begin{align*}
\Theta_\epsilon := \left\{\theta\in\Theta\ \Big| \ \| \tilde{r}_\theta - \tilde{r}_{\theta^{*}} \| > \kappa_\epsilon \vee  \| (\tilde{p}_\theta - \tilde{p}_{\theta^{*}})\tr \tilde{V}^*_{\theta^*} \| > \frac{\kappa_\epsilon}{\gamma} \right\},
\end{align*}
where $\kappa_\epsilon := \frac{(1-\gamma)\epsilon}{4}-\frac{\Delta(1+\gamma)}{2}$. Then, under event $E$, for all $t < \tau$, there exists at least one model $\theta\in\Theta_\epsilon$ such that $\theta\in\bar{\Theta}_t$.
\end{lemma}
\begin{proof}
We note that, for all $t < \tau$, under event $E$, it must be that
\begin{align*}
\exists \theta \in \bar{\Theta}_t, s \in \mathcal{S} :  \tilde{V}_{\theta}^{\tilde{\pi}^*_{\theta^*}}(s) < \tilde{V}_{\theta}^{*}(s) - \epsilon + 2\Delta\frac{(1+\gamma)}{1-\gamma},
\end{align*}
otherwise the algorithm would stop before $\tau$. This implies that $|\tilde{V}_{\theta}^*(s) - \tilde{V}_{\theta}^{\tilde{\pi}^*_{\theta^*}}(s)| > \epsilon - 2\Delta\frac{(1+\gamma)}{1-\gamma}$ holds as well and, using Corollary \ref{cor:v-error-decomposition},
\begin{align}\label{eq:v-error}
    2\max_{\pi \in \{ \tilde{\pi}^*_{\theta^*}, \tilde{\pi}^*_{\theta}\}} \sum_{s',a'} \nu_{\theta}^{\pi}(s',a';s)\left[ |\tilde{r}_\theta(s',a') - \tilde{r}_{\theta^{*}}(s',a')| + \gamma |(\tilde{p}_{\theta}(s',a')-\tilde{p}_{\theta^{*}}(s',a'))\tr \tilde{V}^{*}_{\theta^{*}}|\right] > \epsilon - 2\Delta\frac{(1+\gamma)}{1-\gamma}
\end{align}
holds for some $\theta\in\bar{\Theta}_t$ and $s\in\mathcal{S}$. Assume that all models in $\Theta_\epsilon$ have been eliminated. Then, using that $\nu$ sums up to $1/(1-\gamma)$ and that all models must be sufficiently close to $\theta^*,$ the left-hand side of this inequality can be upper bounded by $\epsilon - 2\Delta\frac{(1+\gamma)}{1-\gamma}$. Hence, we obtain a contradiction and it must be that $\theta\in\tilde{\Theta}_t$ for some $\theta\in\Theta_\epsilon$.
\end{proof}

\begin{lemma}[Positive index]\label{lemma:idx-positive}
Let $\tau$ be the random stopping time of Algorithm \ref{alg:policy-transfer}, then, under event $E$,
\begin{align*}
	\forall t < \tau: \Psi_{t}(S_{t}, A_{t}) > 0 
\end{align*}
\end{lemma}
\begin{proof}
Recall that the algorithm enters the transfer mode if $\Delta < \frac{\epsilon(1-\gamma)}{4(1+\gamma)}$. Take any time $t < \tau$. Under event $E$, we have $\theta^*\in\bar{\Theta}_t$ and Lemma \ref{lemma:theta-eps} implies that $\theta \in \bar{\Theta}_t$ for some $\theta\in\Theta_\epsilon$. The definition of $\Theta_\epsilon$ implies that either $\| \tilde{r}_\theta - \tilde{r}_{\theta^{*}} \| > \kappa_\epsilon$ or $\| (\tilde{p}_\theta - \tilde{p}_{\theta^{*}})\tr \tilde{V}^*_{\theta^*} \| > \frac{\kappa_\epsilon}{\gamma}$ and both these quantities are strictly greater than zero since $\kappa_\epsilon > \frac{\epsilon(1-\gamma)}{8}$. Since the index contains a maximum over models involving these two, the result follows straightforwardly.
\end{proof}

The following lemma is the key result that allows us to bound the sample complexity of Algorithm \ref{alg:policy-transfer}. It shows that, at any time $t$, the number of times the chosen state-action action pair $(S_t,A_t)$ has been chosen before is bounded by a quantity proportional to minimum number of samples required from any state-action pair to eliminate any of the active models.

\begin{lemma}[Fundamental lemma]\label{lemma:fundamental}
Let $(S_{t}, A_{t})$ be the state-action pair chosen at time t. Then, under event E, the number of queries to such couple prior to time $t$ can be upper bounded by
\begin{align*}
N_t(S_{t}, A_{t}) < \frac{128\log (8SAn(|\Theta|+1)/\delta)}{\max_{s,a}\max_{\theta\in\bar{\Theta}_t}\Psi_{s,a}(\theta^*,\theta)}.
\end{align*}
\end{lemma}

\begin{proof}
Let $F_t = \indi{\forall s,a \in \mathcal{S}\times\mathcal{A} : \Psi_t^r(S_t,A_t) \geq \Psi_t(s,a)}$ be the event under which, at time $t$, the maximizer of the index is attained by the reward components. The proof is divided in two parts, based on whether $F_t$ holds or not.

\paragraph{Event $F_t$ holds} 

We start by defining some quantities. Let $\underline{\Theta}_t := \argmax_{\theta,\theta' \in \bar{\Theta}_t} \Psi^{r}_{t}(s,a) $ be the set of active models that attain the maximum in the reward index. Similarly, define
\begin{align*}
    \bar{\theta}_t := \argmax_{\theta \in \bar{\Theta}_t} \adevr{S_t,A_t}{\theta^*,\theta}, \quad \underline{\theta}_t := \argmin_{\theta \in \bar{\Theta}_t} \adevr{S_t,A_t}{\theta^*,\theta},
\end{align*}
as the farthest and closest models from $\theta^*$ among the active ones, respectively. Assume, without loss of generality, that the maximums/minimums are attained by single models. If more than one model attains them, the proof follows equivalently by choosing arbitrary ones. Furthermore, let $\theta_t^v$ be the (random) model among those in $\underline{\Theta}_t$ whose reward-variance is used to attain the maximum in the index.

We now proceed as follows. First, we prove that an upper bound to the index of the chosen state-action pair directly relates to the sample complexity for eliminating $\bar{\theta}_t$. Then, we use this result to guarantee that $(S_t, A_t)$ cannot be chosen more than the stated quantity prior to time step $t$, otherwise $\bar{\theta}_t$ could not be an active model.

By assumption we have
\begin{align*}
    \Psi_t(S_t, A_t) &= \Psi_t^r(S_t, A_t) = \min \left\{ \frac{(\adevr{S_t,A_t}{\bar{\theta}_t, \underline{\theta}_t} -8\Delta)^2}{\tilde{\sigma}_{\theta_t^v}^r(S_t, A_t)^2}, \adevr{S_t,A_t}{\bar{\theta}_t, \underline{\theta}_t}-8\Delta \right\}\\ &\stackrel{(a)}{\leq} \min \left\{ \frac{\left(\adevr{S_t,A_t}{\bar{\theta}_t, \theta^*} + \adevr{S_t,A_t}{\theta^*, \underline{\theta}_t} - 8\Delta \right)^2}{\tilde{\sigma}_{\theta_t^v}^r(S_t, A_t)^2}, \adevr{S_t,A_t}{\bar{\theta}_t, \theta^*} + \adevr{S_t,A_t}{\theta^*, \underline{\theta}_t} - 8\Delta\right\}\\ &\stackrel{(b)}{\leq}  \min \left\{ \frac{(2\adevr{S_t,A_t}{\bar{\theta}_t, \theta^*}-8\Delta)^2}{\tilde{\sigma}_{\theta_t^v}^r(S_t, A_t)^2}, 2\adevr{S_t,A_t}{\bar{\theta}_t, \theta^*}-8\Delta\right\} \\ &\leq 4\min \left\{ \frac{(\adevr{S_t,A_t}{\bar{\theta}_t, \theta^*}-4\Delta)^2}{\tilde{\sigma}_{\theta_t^v}^r(S_t, A_t)^2}, \adevr{S_t,A_t}{\bar{\theta}_t, \theta^*}-4\Delta\right\},
\end{align*}
where (a) follows from the triangle inequality and (b) from the definition of $\bar{\theta}_t$ (which was defined as the farthest from the estimate of $\theta^*$). Note that to prove this inequalities we also need that $\adevr{S_t,A_t}{\bar{\theta}_t, \underline{\theta}_t} - 8\Delta \geq 0$, which is implied by Lemma \ref{lemma:idx-positive}. Since $(S_t, A_t)$ is chosen at time $t$, it must be that $\Psi_t(S_t, A_t) \geq \Psi_t(s,a)$ for all $s,a \in \mathcal{S}\times\mathcal{A}$. This implies that, for all $s,a \in \mathcal{S}\times\mathcal{A}$ and $\theta\in\tilde{\Theta}_t$,
\begin{align}\label{eq:idx-r-lb}
    4\min \left\{ \frac{(\adevr{S_t,A_t}{\bar{\theta}_t, \theta^*}-4\Delta)^2}{\tilde{\sigma}_{\theta_t^v}^r(S_t, A_t)^2}, \adevr{S_t,A_t}{\bar{\theta}_t, \theta^*}-4\Delta\right\} &\geq \Psi_t(s,a) \geq \Psi_{s,a}(\theta^*,\theta),
\end{align}
where the second inequality holds since the index of $(s,a)$ is by definition larger than the one using the models $\theta^*$ and $\theta$. Note that Lemma \ref{lemma:model-elim} and \ref{lemma:idx-positive} ensure that a number of queries to $(S_t,A_t)$ of
\begin{align*}
    32\log\frac{8SAn(|\Theta|+1)}{\delta} \max \left\{ \frac{\tilde{\sigma}_{\theta_t^v}^r(S_t, A_t)^2}{(\adevr{S_t,A_t}{\bar{\theta}_t, \theta^*}-4\Delta)^2}, \frac{1}{\adevr{S_t,A_t}{\bar{\theta}_t, \theta^*}-4\Delta}\right\}
\end{align*}
suffices for eliminating $\bar{\theta}_t$. In particular, Lemma \ref{lemma:idx-positive} implies that $\adevr{S_t,A_t}{\bar{\theta}_t, \underline{\theta}_t} > 8\Delta$, which, in turn, implies that $\adevr{S_t,A_t}{\theta^*, \bar{\theta}_t} > 4\Delta$.
Therefore, Equation \ref{eq:idx-r-lb} above implies that a number of queries of
\begin{align*}
    \frac{128\log (8SAn(|\Theta|+1)/\delta)}{\max_{s,a}\max_{\theta\in\bar{\Theta}_t}\Psi_{s,a}(\theta^*,\theta)}.
\end{align*}
also suffices. We note that the maximums at the denominator can be introduced since \eqref{eq:idx-r-lb} holds for all $s,a$ and $\theta$. Hence, it must be that $N_{t}(S_t, A_t)$ is strictly less than this quantity, otherwise the model $\bar{\theta}_t$ would be eliminated at time step $t-1$ and it could not be active at time $t$. This concludes the first part of the proof.

\paragraph{Event $F_t$ does not hold}

In this case, the maximizer of the index must be attained using the transition components, thus $\Psi_t(S_t, A_t) = \Psi_t^p(S_t, A_t)$. The proof follows exactly the same steps as before and is therefore not reported. Since the result is the same, combining these two parts proves the main statement.

\end{proof}

The following lemma ensures that Algorithm \ref{alg:policy-transfer} returns $\epsilon$-optimal policies with high probability.

\begin{lemma}[Correctness]\label{lemma:err-dec}
Let $\tau$ be the stopping time of Algorithm \ref{alg:policy-transfer} and ${\pi}_{\tau}$ be the returned policy. Then, under event $E$, ${\pi}_{\tau}$ is $\epsilon$-optimal with respect to $\theta^*$. 
\end{lemma}
\begin{proof}
Recall that ${\pi}_\tau$ is $\epsilon$-optimal if, for all states, $V^{\pi_\tau}_{\theta^*}(s) \geq V^*_{\theta^*}(s) - \epsilon$. Furthermore, $\pi_\tau$ is optimal for one of the active models at time $\tau$, i.e., $\pi_\tau = \tilde{\pi}^*_\theta$ for some $\theta \in \bar{\Theta}_{\tau}$. Since under $E$ we have $\theta^* \in \bar{\Theta}_\tau$, a sufficient condition is that $\|V_{\theta'}^{*} - V_{\theta'}^{\tilde{\pi}^{*}_{\theta}} \| < \epsilon$ holds for all $\theta' \in \bar{\Theta}_\tau$. Let us upper bound the left-hand side as
\begin{align*}
\|V_{\theta'}^{*} - V_{\theta'}^{\tilde{\pi}^{*}_{\theta}} \| \le \|\tilde{V}_{\theta'}^{\tilde{\pi}^{*}_{\theta}}-V_{\theta'}^{\tilde{\pi}^{*}_{\theta} }  \| + \|V^{*}_{\theta'} - \tilde{V}_{\theta'}^{*} \| + \|\tilde{V}_{\theta'}^{\tilde{\pi}^{*}_{\theta}} - \tilde{V}_{\theta'}^{*}\|
\end{align*}
Using Lemma \ref{lemma:simulation}, we can bound the first term by
\begin{align*}
\|\tilde{V}_{\theta'}^{\tilde{\pi}^{*}_{\theta}}-V_{\theta'}^{\tilde{\pi}^{*}_{\theta} }  \| &\le \sum_{s', a'} \nu^{\tilde{\pi}^{*}_{\theta}}_{\theta'} (s', a'; s) ( |r_{\theta'}(s,a) - \tilde{r}_{\theta'}(s,a)| + \gamma |(p_{\theta'}(s,a) - \tilde{p}_{\theta'}(s,a))\tr \tilde{V}_{\theta'}^{\tilde{\pi}^{*}_{\theta}}| \\ &\leq \sum_{s', a'} \nu^{\tilde{\pi}^{*}_{\theta}}_{\theta'} (s', a'; s) (\Delta + \gamma \Delta)  \le \frac{\Delta (1+\gamma)}{1-\gamma} 
\end{align*}
and the second term by
\begin{align*}
\|V^{*}_{\theta'} - \tilde{V}_{\theta'}^{*} \| &\le \max_{\pi \in \{ \pi_{\theta'}^{*}, \tilde{\pi}_{\theta'}^{*} \}} \sum_{s', a'} \nu^{\pi}_{\theta'} (s', a'; s) ( |r_{\theta'}(s,a) - \tilde{r}_{\theta'}(s,a)| + \gamma |(p_{\theta'}(s,a) - \tilde{p}_{\theta'}(s,a))\tr \tilde{V}^{*}_{\theta'}| \\ &\leq \max_{\pi \in \{ \pi_{\theta'}^{*}, \tilde{\pi}_{\theta'}^{*} \}} \sum_{s', a'} \nu^{\pi}_{\theta'} (s', a'; s) (\Delta + \gamma \Delta)   \le \frac{\Delta (1+\gamma)}{1-\gamma}.
\end{align*}
Therefore, the stopping condition,
\begin{align*}
 \|\tilde{V}_{\theta'}^{\tilde{\pi}^{*}_{\theta}} - \tilde{V}_{\theta'}^{*}\| + 2\Delta\frac{(1+\gamma)}{1-\gamma} \leq \epsilon
\end{align*}
implies that $\|V_{\theta'}^{*} - V_{\theta'}^{\tilde{\pi}^{*}_{\theta}} \| \leq \epsilon$, which in turn implies the $\epsilon$-optimality of $\pi_\tau$.
\end{proof}

\subsection{Sample Complexity Bounds}

We are now ready to prove the main theorem, which bounds the sample complexity of Algorithm \ref{alg:policy-transfer}.

\sampcomp*

\begin{proof}
Lemma \ref{lemma:conf} ensures that event $E$ holds with probability at least $1-\delta$. Therefore, we shall carry out the proof conditioned on $E$. \\
We split the proof into two parts. In the first one, we bound the number of times each state-action pair can be visited before the algorithm stops. In the second part, we directly bound the number of steps in which each model can be active.

\paragraph{Bound over $\mathcal{S}\times\mathcal{A}$}
Take any state-action pair $(s,a) \in \mathcal{S}\times\mathcal{A}$. For any sequence $\{ n_{t} \}_{t \ge 1}$, its number of visits can be written as
\begin{align*}
N_\tau(s,a) &= \sum_{t=1}^\tau \indi{S_t = s \wedge A_t = a | E} \\ & = \underbrace{\sum_{t=1}^\tau \indi{S_t = s \wedge A_t = a \wedge N_t(s,a) < n_t | E}}_{(a)} + \underbrace{\sum_{t=1}^\tau \indi{S_t = s \wedge A_t = a \wedge N_t(s,a) \geq n_t | E}}_{(b)}.
\end{align*}
For
\begin{align*}
    n_t := \frac{128\log (8SAn(|\Theta|+1)/\delta)}{\max_{s,a}\max_{\theta\in\bar{\Theta}_t}\Psi_{s,a}(\theta^*,\theta)},
\end{align*}
Lemma \ref{lemma:fundamental} ensures that, under event $E$, $(b) = 0$. Thus, we only need to bound (a). For all $t < \tau$, Lemma \ref{lemma:theta-eps} implies that there exists a model $\theta\in\Theta_\epsilon$ which also belongs to the confidence set at time $t$, $\theta\in\bar{\Theta}_t$. Therefore, for all $t < \tau$ and $(s',a') \in \mathcal{S}\times\mathcal{A}$, we have $\max_{\theta\in\bar{\Theta}_t}\Psi_{s',a'}(\theta^*,\theta) \geq \min_{\theta\in\Theta_\epsilon}\Psi_{s',a'}(\theta^*,\theta)$. Since we removed all random quantities from $n_t$, we can now bound (a) as
\begin{align}\label{eq:bound-sa}
(a) < \frac{128\log (8SAn(|\Theta|+1)/\delta)}{\max_{s,a}\min_{\theta\in\Theta_\epsilon}\Psi_{s,a}(\theta^*,\theta)}.
\end{align}
This immediately yields a bound on the stopping time,
\begin{align*}
\tau = \sum_{s,a} N_{\tau}(s,a) < \frac{128SA\log (8SAn(|\Theta|+1)/\delta)}{\max_{s,a}\min_{\theta\in\Theta_\epsilon}\Psi_{s,a}(\theta^*,\theta)}.
\end{align*}

\paragraph*{Bound over $\Theta$}

From the first part of the proof, we know that, for $t < \tau$, the confidence set $\bar{\Theta}_t$ must contain a model that is also in $\Theta_\epsilon$, otherwise the algorithm would stop. Therefore, the stopping time can be bounded by
\begin{align*}
\tau \leq \sum_{t=1}^n \indi{\exists \theta \in \bar{\Theta}_t : \theta \in \Theta_\epsilon | E}.
\end{align*}
By definition of the algorithm, the state-action pair chosen at each time step does not change until the set of active models $\underline{\Theta}_t$ (those that control the maximizer of the index as in the proof of Lemma \ref{lemma:fundamental}) does not change. Furthermore, once a model has been eliminated, it cannot become active again. Consider a sequence $\{\tau_h\}_{h\geq 1}$ with $\tau_1=1$. We can partition the time line into different contiguous intervals (from now on called phases) $\mathcal{T}_h := [\tau_h, \tau_{h+1} - 1]$ such that the set of active models does not change within $\mathcal{T}_h$ and a change of phase occurs only when a model is eliminated. Let $\underline{\Theta}_h$ be the set of active models in phase $h$. We have $\tau_{h+1} = \inf_{t \geq 1}\{ t\ | \ \exists \theta\in\underline{\Theta}_h : \theta \notin \bar{\Theta}_t \}$. That is, the beginning of the new phase $h+1$ is the step where one of the previously-active models is eliminated. Let $\bar{\theta}_h$ be any such model and $h(t)$ be the (unique) phase containing time $t$. Note that, for each $\theta\in\Theta\setminus\{\theta^*\}$, there exists at most one phase $h(\theta)$ where $\bar{\theta}_{h(\theta)} = \theta$. Then,
\begin{align*}
\tau &\leq \sum_{\theta\in\Theta \setminus\{\theta^*\}}\sum_{t=1}^n \indi{\bar{\theta}_{h(t)} = \theta \wedge \exists \theta' \in \bar{\Theta}_t : \theta' \in \Theta_\epsilon | E}\\ &\leq \sum_{\theta\in\Theta \setminus\{\theta^*\}}\sum_{t=\tau_{h(\theta)}}^{\tau_{h(\theta) + 1}-1} \indi{\bar{\theta}_{h(t)} = \theta \wedge \exists \theta' \in \bar{\Theta}_t : \theta' \in \Theta_\epsilon | E} \leq \frac{128(|\Theta| - 1)\log (8SAn(|\Theta|+1)/\delta)}{\max_{s,a}\min_{\theta\in\Theta_\epsilon}\Psi_{s,a}(\theta^*,\theta)},
\end{align*}
where in the last inequality we applied Lemma \ref{lemma:fundamental} by noticing that, within the same phase, the chosen state-action pair does not change and used the fact that a model in $\Theta_\epsilon$ still survives to upper bound the minimum over models in the confidence set. The proof follows by taking the minimum of the two bounds.

\end{proof}

\corcomp*

\begin{proof}
We notice that each model $\theta\in\Theta_\epsilon$ is, by definition, such that either $\| \tilde{r}_\theta - \tilde{r}_{\theta^{*}} \| \geq \max \{ \Gamma, \kappa_\epsilon \}$ or $ \| (\tilde{p}_\theta - \tilde{p}_{\theta^{*}})\tr \tilde{V}^*_{\theta^*} \| \geq \max \{ \Gamma, \kappa_\epsilon \}$. By the transfer condition, we also have that $\kappa_\epsilon \geq \frac{(1-\gamma)\epsilon}{8}$. Then, it is easy to see that
\begin{align*}
    \max_{s,a}\min_{\theta\in\Theta_\epsilon}\Psi_{s,a}(\theta^*,\theta) \geq \max \{ \Gamma^2, \kappa_\epsilon^2 \} (1-\gamma)^2 \geq \frac{1}{8}\max \{ \Gamma^2, \epsilon^2 \} (1-\gamma)^4,
\end{align*}
where we use the previous lower bounds and upper bounded the value-function variance by $1/(1-\gamma)^2$. Then, the result follows by rewriting in $\tilde{\mathcal{O}}$ notation.

\end{proof}

\section{Learning HMMs by Tensor Decomposition}\label{app:rtp}

After reducing our setting to a HMM learning problem, we can almost immediately plug the agent's observations into the tensor decomposition approach of \citet{anandkumar2014tensor} and obtain estimates $\hat{O}$ and $\hat{T}$ of the desired matrices. We now briefly describe how this method works as some of its features are needed for our analysis later. The detailed steps are reported in Algorithm \ref{alg:hmm-learning}. The key intuition behind the method of \citet{anandkumar2014tensor} is that the second and third moments of the HMM observations posses a low-rank tensor structure. More precisely, it is possible to find a transformation of these observations such that the resulting third moment is a symmetric and orthogonal tensor whose spectral decomposition directly yields (transformations of) the HMM parameters. To see this, we first formulate the problem as an instance of a multi-view model \cite{sun2013survey}. Take three consecutive observations (our "views"), say $o_1, o_2, o_3$, and let $\Sigma_{i,j} := \expec{o_i \otimes o_j}$, for $i,j \in \{1,2,3\}$, be their covariance matrices, where $\otimes$ denotes the tensor product. Define the transformed views as $\tilde{o}_1 = \Sigma_{3,2} \Sigma_{1,2}^\dagger o_1$ and $\tilde{o}_2 = \Sigma_{3,1} \Sigma_{2,1}^\dagger o_2$, and let the second and third cross-view moments be $M_2 = \expec{\tilde{o}_1 \otimes \tilde{o}_2}$ and $M_3 = \expec{\tilde{o}_1 \otimes \tilde{o}_2 \otimes o_3}$, respectively. Then, Theorem 3.6 of \citet{anandkumar2014tensor} shows that $M_2 = \sum_{j=1}^k \omega_j \mu_{3,j} \otimes \mu_{3,j}$ and $M_3 = \sum_{j=1}^k \omega_j \mu_{3,j} \otimes \mu_{3,j} \otimes \mu_{3,j}$, where $\mu_{3,j} = \expec{o_3 | \theta_3^* = \theta_j}$ and $\omega_j = \omega(\theta_j)$. Hence, these moments posses a low-rank tensor structure, as they can be decomposed into tensor products of vectors, and it is possible to recover the conditional means $\mu_{3,j}$ using the robust tensor power (RTP) method of \citet{anandkumar2014tensor}. Given these, from Proposition 4.2 of \citet{anandkumar2012method}, we can recover the HMM parameters as follows:
\begin{align*}
     [O]_{:,j} &= \Sigma_{2,1} \Sigma_{3,1}^\dagger \mu_{3,j},\\
    [T]_{:,j} &= O^\dagger \mu_{3,j}.
\end{align*}
In practice, all the required moments can be estimated from samples. Suppose the agent has observed $3m$ tasks and split these in $m$ triples of contiguous tasks. Then, we can estimate the covariance matrices between views as $\hat{\Sigma}_{i,j} = \frac{1}{m} \sum_{h=1}^m o_{i,h} \otimes o_{j,h}$, and similarly for the second and third cross moments of the transformed observations.

\begin{algorithm}[t]
\caption{Learning HMMs by Tensor Decomposition} \label{alg:hmm-learning}
\begin{algorithmic}[1]
\REQUIRE Observations $\{o_l\}_{l=1}^h$ with $o_l \in \R^d$, number of tasks $k$
\ENSURE Estimated observation matrix $\hat{O}$ and transition matrix $\hat{T}$

\STATE{Split observations into $m = \lfloor h / 3 \rfloor$ triples: $\{(o_{1,l}, o_{2,l}, o_{3,l})\}_{l=1}^m$}
\STATE{Estimate covariance matrices: $\hat{\Sigma}_{i,j} = \frac{1}{m} \sum_{l=1}^m o_{i,l} \otimes o_{j,l}$, for $i,j \in \{1,2,3\}$}
\STATE{Get transformed observations: $\tilde{o}_{1,l} = \hat{\Sigma}_{3,2} \hat{\Sigma}_{1,2}^\dagger o_{1,l},\ \ \tilde{o}_{2,l} = \hat{\Sigma}_{3,1} \hat{\Sigma}_{2,1}^\dagger o_{2,l},\ l\in[m]$}
\STATE{Estimate 2\textsuperscript{nd} and 3\textsuperscript{nd} moments: $\hat{M}_2 = \frac{1}{m} \sum_{l=1}^m \tilde{o}_{1,l} \otimes \tilde{o}_{2,l}$, $\ \ \hat{M}_3 = \frac{1}{m} \sum_{l=1}^m \tilde{o}_{1,l} \otimes \tilde{o}_{2,l} \otimes o_{3,l}$}
\STATE{Find the $k$ largest eigenvectors $\hat{D}\in\R^{d\times k}$ and eigenvalues $\hat{\Lambda}\in\R^{k\times k}$ of $\hat{M}_2$}
\STATE{Compute the whitening matrix: $\hat{W} = \hat{D} \hat{\Lambda}^{-\frac{1}{2}}$}
\STATE{Run Algorithm 1 of \citet{anandkumar2014tensor} on $\hat{M}_3(\hat{W}, \hat{W}, \hat{W})$, obtain estimated eigenvalues $\{\hat{\lambda}_j\}_{j=1}^k$ and eigenvectors $\{\hat{v}_j\}_{j=1}^k$}
\STATE{Estimate conditional means of the third view: $\hat{\mu}_{3,j} = \hat{\lambda}_j (\hat{W}^T)^\dagger \hat{v}_j$}
\STATE{Estimate mean observations: $[\hat{O}]_{:,j} = \hat{\Sigma}_{2,1} \hat{\Sigma}_{3,1}^\dagger \hat{\mu}_{3,j}$}
\STATE{Estimate transition matrix: $[\hat{T}]_{:,j} = \hat{O}^\dagger \hat{\mu}_{3,j}$}

\end{algorithmic}
\end{algorithm}

\section{Analysis of the Sequential Transfer Algorithm}\label{app:proof4}

\subsection{Definitions and Assumptions}

In this section, we analyze the approximate models computed by the sequential transfer algorithm (Algorithm \ref{alg:sequential-transfer}) using the RTP method (Algorithm \ref{alg:hmm-learning}). To simplify notation, we drop the task index $h$ whenever clear from the context.

We introduce some additional vector/matrix notation. We use $\|\cdot\|_2$ to denote the spectral norm and $\| \cdot \|_F$ to denote the Frobenius norm. For a matrix $A$, we denote by $\lambda_{\mathrm{min}}(A)$ and $\lambda_{\mathrm{max}}(A)$ its minimum and maximum singular value, respectively.

We make use of the following assumption in order to bound the estimation error of the reward and transition variances.
\begin{assumption}
 There exists a positive constant $\underline{\sigma} > 0$ which lower-bounds both the true reward and value variances, ${\sigma}_{\theta}^r(s,a)$ and ${\sigma}_{\theta}^p(s,a; \theta')$, and their estimates $\tilde{\sigma}_{h, \theta}^r(s,a)$ and $\tilde{\sigma}_{h, \theta}^p(s,a; \theta')$ at any $h$.
\end{assumption}
We believe this assumption could be removed, at least for the approximate models, but at the cost of more complicated proofs.

\subsection{Supporting Lemmas}

\begin{lemma}[Lemma 5 of \citet{azizzadenesheli2016reinforcement}]\label{lemma:bound-aziz}
Let $\{\hat{\mu}_{3,j}\}_{j=1}^k$ be the columns of the third-view matrix estimated by the RTP method (Algorithm \ref{alg:hmm-learning} in Appendix \ref{app:rtp}) after observing $h$ tasks. Then, there exists two constants $\rho_1(\Theta, T)$, $\rho_2(\Theta, T)$ such that, for any $\delta' \in (0,1)$, if
\begin{align*}
    h > \rho_1(\Theta, T) \log\frac{ 2SA(S+U)}{\delta'},
\end{align*}
then, under Assumption \ref{ass:full-rank}, with probability at least $1-\delta'$ and up to some permutation of the columns of the third view,
\begin{align*}
    \| \hat{\mu}_{3,j} - \mu_{3,j} \|_2 \leq \rho_2(\Theta, T)\sqrt{\frac{\log\left(2SA(S+U) / \delta'\right)}{h}}.
\end{align*}
\end{lemma}

In Lemma \ref{lemma:bound-aziz}, compared to the original result, we collapsed all terms of minor relevance for our purpose into two constants $\rho_1,\rho_2$. These are functions of the given family of tasks (through maximum/minimum eigenvalues of the covariance matrices introduced before) and of the underlying Markov chain. We refer the reader to Appendix C of \citet{azizzadenesheli2016reinforcement} for their full expression.

We now bound the estimation error of the observation matrix $O$, which will be directly used to bound the errors of the approximate MDP models.

\begin{lemma}[Estimation Error of $O$]\label{lemma:O-error}
Let $\hat{O}$ be the observation matrix estimated by Algorithm \ref{alg:hmm-learning} using $h$ tasks, $\delta'\in(0,1)$, and 
\begin{align*}
     \rho_3(\Theta, T) :=  \max \left\{ \rho_1(\Theta, T) ,  \frac{1}{\maxsv{\Sigma_{2,1}}^2}, \frac{4}{\minsv{\Sigma_{3,1}}^2} \right\}.
\end{align*}
Then, if $h \geq \rho_3(\Theta, T)\log (6SA(S+U)/\delta')$, we have that, with probability at least $1-\delta'$,
\begin{align*}
    \left\| [O]_{:,j} - [\hat{O}]_{:,j} \right\|_2 \leq \rho_4(\Theta, T)\sqrt{\frac{\log\left(6SA(S+U) / \delta'\right)}{h}},
\end{align*}
where 
\begin{align*}
    \rho_4(\Theta, T) := \frac{4\maxsv{\Sigma_{2,1}}\rho_2(\Theta,T) + 4SA + 8SA\maxsv{\Sigma_{2,1}}}{\minsv{\Sigma_{3,1}}}.
\end{align*}
\end{lemma}
\begin{proof}
Recall that $[O]_{:,j} = \Sigma_{2,1} \Sigma_{3,1}^\dagger \mu_{3,j}$ and similarly for its estimate $[\hat{O}]_{:,j} = \hat{\Sigma}_{2,1} \hat{\Sigma}_{3,1}^\dagger \hat{\mu}_{3,j}$. Let us decompose the total error into the deviations of each single component,
\begin{align*}
    \left\| [O]_{:,j} - [\hat{O}]_{:,j} \right\|_2 &= 
    \left\| \Sigma_{2,1} \Sigma_{3,1}^\dagger \mu_{3,j} - \hat{\Sigma}_{2,1} \hat{\Sigma}_{3,1}^\dagger \hat{\mu}_{3,j} \pm \hat{\Sigma}_{2,1} \hat{\Sigma}_{3,1}^\dagger {\mu}_{3,j} \right\|_2\\ &\leq
    \left\| \hat{\Sigma}_{2,1}\right\|_2 \left\| \hat{\Sigma}_{3,1}^\dagger \right\|_2 \left\| \mu_{3,j} - \hat{\mu}_{3,j} \right\|_2 + \left\| \Sigma_{2,1} \Sigma_{3,1}^\dagger \mu_{3,j} - \hat{\Sigma}_{2,1} \hat{\Sigma}_{3,1}^\dagger {\mu}_{3,j} \pm \hat{\Sigma}_{2,1} {\Sigma}_{3,1}^\dagger {\mu}_{3,j} \right\|_2\\ &\leq
    \left\| \hat{\Sigma}_{2,1}\right\|_2 \left\| \hat{\Sigma}_{3,1}^\dagger \right\|_2 \left\| \mu_{3,j} - \hat{\mu}_{3,j} \right\|_2 + \left\| \Sigma_{3,1}^\dagger \right\|_2 \left\| \mu_{3,j}\right\|_2 \left\| \Sigma_{2,1} - \hat{\Sigma}_{2,1} \right\|_2 \\ & \hspace{5.06cm} + \left\| \hat{\Sigma}_{2,1} \right\|_2 \left\| \mu_{3,j}\right\|_2 \left\| \Sigma_{3,1}^\dagger - \hat{\Sigma}_{3,1}^\dagger \right\|_2.
\end{align*}
We now bound all these components separately. We first notice that, from Proposition 6 of \citet{azizzadenesheli2016reinforcement},
\begin{align}\label{eq:prop-6-aziz}
    \left\| \Sigma_{2,1} - \hat{\Sigma}_{2,1} \right\|_2 \leq \sqrt{\frac{\log 1/\delta'}{h}}, \quad \left\| \Sigma_{3,1} - \hat{\Sigma}_{3,1} \right\|_2 \leq \sqrt{\frac{\log 1/\delta'}{h}},
\end{align}
Using this result,
\begin{align*}
    \left\| \hat{\Sigma}_{2,1}\right\|_2 &\leq \left\| {\Sigma}_{2,1}\right\|_2 + \left\| {\Sigma}_{2,1} - \hat{\Sigma}_{2,1}\right\|_2 \leq \maxsv{{\Sigma}_{2,1}} + \left\| {\Sigma}_{2,1} - \hat{\Sigma}_{2,1}\right\|_2\\ &\leq \maxsv{{\Sigma}_{2,1}} + \sqrt{\frac{\log 1/\delta'}{h}} \leq 2\maxsv{{\Sigma}_{2,1}},
\end{align*}
which holds for
\begin{align}\label{eq:cond-m-1}
    h \geq \frac{\log 1/\delta'}{\maxsv{\Sigma_{2,1}}^2}.
\end{align}
Using Lemma E.1 of \citet{anandkumar2012method},
\begin{align*}
    \left\| \hat{\Sigma}_{3,1}^\dagger \right\|_2 \leq \frac{1}{\minsv{\hat{\Sigma}_{3,1}}} \leq \frac{2}{\minsv{{\Sigma}_{3,1}}},
\end{align*}
which holds when
\begin{align*}
    \frac{\| \Sigma_{3,1} - \hat{\Sigma}_{3,1}\|}{\minsv{\Sigma_{3,1}}} \leq \frac{1}{2}.
\end{align*}
From \eqref{eq:prop-6-aziz}, a sufficient condition for this to hold is
\begin{align}\label{eq:cond-m-2}
    h \geq \frac{4\log 1/\delta'}{\minsv{\Sigma_{3,1}}^2}.
\end{align}
Under this same condition, we can apply Proposition 7 of \citet{azizzadenesheli2016reinforcement} to bound
\begin{align*}
    \left\| \Sigma_{3,1}^\dagger - \hat{\Sigma}_{3,1}^\dagger \right\|_2 \leq \frac{2}{\minsv{\Sigma_{3,1}}}\sqrt{\frac{\log 1/\delta'}{h}}.
\end{align*}
Since the columns $\mu_{3,j}$ of the third-view sum up to one every $S$ components (for the transition model) and every $U$ components (for the reward model), we have $\| \mu_{3,j} \|_2 \leq \| \mu_{3,j} \|_1 \leq 2SA$. Finally, we can bound the error in estimating such columns using Lemma \ref{lemma:bound-aziz},
\begin{align*}
    \| \hat{\mu}_{3,j} - \mu_{3,j} \|_2 \leq \rho_2(\Theta, T)\sqrt{\frac{\log\left(2SA(S+U) / \delta'\right)}{h}}.
\end{align*}
Plugging everything into the initial error decomposition and rearranging,
\begin{align*}
    \left\| [O]_{:,j} - [\hat{O}]_{:,j} \right\|_2 \leq \frac{4\maxsv{\Sigma_{2,1}}\rho_2(\Theta,T) + 2SA + 4SA\maxsv{\Sigma_{2,1}}}{\minsv{\Sigma_{3,1}}}\sqrt{\frac{\log\left(2SA(S+U) / \delta'\right)}{h}},
\end{align*}
which holds for when $h$ satisfies the conditions of Lemma \ref{lemma:bound-aziz}, Equation \ref{eq:cond-m-1}, and Equation \ref{eq:cond-m-2}. The two bounds of \eqref{eq:prop-6-aziz} and that of Lemma \ref{lemma:bound-aziz} hold each with probability at least $1-\delta'$, hence the final bound holds with probability at least $1-3\delta'$. Renaming $\delta'$ into $3\delta'$ concludes the proof.
\end{proof}

Similarly to $O$, we also bound the estimation error of $T$. The proof follows Theorem 3 of \citet{azizzadenesheli2016reinforcement}.

\begin{lemma}[Estimation error of $T$]\label{lemma:err-T}
Let $\hat{T}$ be the task-transition matrix estimated by Algorithm \ref{alg:hmm-learning} using $h$ samples and
\begin{align*}
    \rho_5(\Theta,T) := \max \left\{ \rho_3(\Theta,T), \frac{4k\rho_4(\Theta,T)^2}{\lambda_k(O)^2} \right\}.
\end{align*}
Then, for any $\delta' \in (0,1)$, if $h \geq \rho_5(\Theta, T)\log (6SA(S+U)/\delta')$, we have that, with probability at least $1-\delta'$,
\begin{align*}
    \left\| [T]_{:,j} - [\hat{T}]_{:,j} \right\|_2 \leq \rho_6(\Theta,T)\sqrt{\frac{\log\left(6SA(S+U) / \delta'\right)}{h}},
\end{align*}
where
\begin{align*}
    \rho_6(\Theta,T) := 8SA\sqrt{k}\frac{1 + \sqrt{5}}{\lambda_k(O)^2}\rho_4(\Theta,T).
\end{align*}
\end{lemma}
\begin{proof}
Recall that $[T]_{:,j} = O^\dagger \mu_{3,j}$ and $[\hat{T}]_{:,j} = \hat{O}^\dagger \hat{\mu}_{3,j}$. Similarly to Lemma \ref{lemma:O-error}, we decompose the error into
\begin{align*}
    \left\| [T]_{:,j} - [\hat{T}]_{:,j} \right\|_2 = \left\| O^\dagger \mu_{3,j} - \hat{O}^\dagger \hat{\mu}_{3,j} \right\|_2 \leq \left\| \mu_{3,j} \right\|_2\left\| O^\dagger - \hat{O}^\dagger\right\|_2 + \left\| \hat{O}^\dagger \right\|_2 \left\| \mu_{3,j} - \hat{\mu}_{3,j} \right\|_2.
\end{align*}
In the proof of Lemma \ref{lemma:O-error} we already bounded $\| \mu_{3,j} \|_2 \leq 2SA$, while the term $\left\| \mu_{3,j} - \hat{\mu}_{3,j} \right\|_2$ was bounded in Lemma \ref{lemma:bound-aziz}. Let us bound the remaining two terms. Take $\lambda_k(\hat{O})$ as the $k$-th singular value of $\hat{O}$. Then, following \citet{azizzadenesheli2016reinforcement},
\begin{align*}
     \left\| \hat{O}^\dagger \right\|_2 \leq \frac{1}{\sigma_k(\hat{O})} \leq \frac{2}{\sigma_k({O})},
\end{align*}
where the second inequality follows from Lemma E.1 of \citet{anandkumar2012method} under the assumption that
\begin{align}\label{eq:cond-o}
    \frac{\| O - \hat{O}\|_2}{\sigma_k(O)} \leq \frac{1}{2}.
\end{align}
Lemma \ref{lemma:O-error} already bounds the $l_2$ error in the columns of $O$. Therefore,
\begin{align*}
    \| O - \hat{O}\|_2 \leq \| O - \hat{O}\|_F \leq \sqrt{k}\max_{j\in [k]}\| O_j - \hat{O}_j\|_2 \leq \rho_4(\Theta,T)\sqrt{\frac{k\log\left(6SA(S+U) / \delta'\right)}{h}}.
\end{align*}
Therefore, a sufficient condition for \eqref{eq:cond-o} is
\begin{align*}
    m \geq \frac{4k\rho_4(\Theta,T)^2\log\left(6SA(S+U) / \delta'\right)}{\lambda_k(O)^2}.
\end{align*}
In order to bound the deviation of the pseudo-inverse of $O$, we apply Theorem 1.1 of \citet{meng2010optimal},
\begin{align*}
    \left\| O^\dagger - \hat{O}^\dagger\right\|_2 \leq \frac{1 + \sqrt{5}}{2} \max \left\{ \left\| {O}^\dagger \right\|_2^2, \left\| \hat{O}^\dagger \right\|_2^2 \right\}\| O - \hat{O}\|_2 \leq \frac{2 + 2\sqrt{5}}{\lambda_k(O)^2}\rho_4(\Theta,T)\sqrt{\frac{k\log\left(6SA(S+U) / \delta'\right)}{h}}.
\end{align*}
Finally, plugging everything back into the first error decomposition,
\begin{align*}
    \left\| [T]_{:,j} - [\hat{T}]_{:,j} \right\|_2 \leq 8SA\frac{1 + \sqrt{5}}{\lambda_k(O)^2}\rho_4(\Theta,T)\sqrt{\frac{k\log\left(6SA(S+U) / \delta'\right)}{h}}.
\end{align*}
\end{proof}

We also need the following technical lemma which bounds the difference in standard deviations between random variables under different distributions.

\begin{lemma}\label{lemma:var-bound}
Let $f$ be a function which takes values in $[0, b]$, for some $b > 0$, and $p,q \in \mathcal{X}$ two probability distributions on a finite set $\mathcal{X}$. Denote by $\Var_p[f]$ and $\Var_q[f]$ the variance of $f$ under $p$ and $q$, respectively, and assume both to be larger than some constant $c > 0$. Then,
\begin{align*}
    \left| \sqrt{\Var_p[f]} - \sqrt{\Var_q[f]} \right| \leq b\left(\frac{b}{2\sqrt{c}} + 1 \right)\| p - q \|_1,
\end{align*}
\end{lemma}
\begin{proof}
Let $\mu_p := \E_p[f]$ and $\mu_q := \E_q[f]$. For clarity, rewrite the standard deviations as
\begin{align*}
    \sqrt{\Var_p[f]} = \sqrt{\sum_{x\in\mathcal{X}} p(x) (f(x) - \mu_p)^2} = \| f - \mu_p\|_{2,p},
\end{align*}
and similarly for $q$. Here $\|\cdot\|_{2,p}$ denotes the $l_2$-norm weighted by $p$. Then,
\begin{align}\label{eq:var-decomp}
    \left| \| f - \mu_p\|_{2,p} - \| f - \mu_q\|_{2,q}  \right| \leq \left| \| f - \mu_p\|_{2,p} - \| f - \mu_p\|_{2,q}  \right| + \left| \| f - \mu_p\|_{2,q} - \| f - \mu_q\|_{2,q}  \right|.
\end{align}
Let us bound these two terms separately. For the second one, a direct application of Minskowsky's inequality yields,
\begin{align*}
    \| f - \mu_p\|_{2,q} \leq \| \mu_p - \mu_q\|_{2,q} + \| f - \mu_q\|_{2,q} = | \mu_p - \mu_q | + \| f - \mu_q\|_{2,q}.
\end{align*}
Therefore, applying the same reasoning to the other side, we obtain 
\begin{align*}
    \left| \| f - \mu_p\|_{2,q} - \| f - \mu_q\|_{2,q}  \right| \leq | \mu_p - \mu_q | \leq b\| p - q \|_1.
\end{align*}
We now take care of the first term. Since we have a term of the form $| \sqrt{x} -\sqrt{y} |$ and the concavity of the square root implies $| \sqrt{x} -\sqrt{y} | \leq \frac{1}{2}\max\{ \frac{1}{\sqrt{x}}, \frac{1}{\sqrt{y}}\}|x-y|$, we can reduce the problem to bounding the difference of variances,
\begin{align*}
    \left| \| f - \mu_p\|_{2,p}^2 - \| f - \mu_p\|_{2,q}^2  \right| \leq b^2 \| p - q \|_1,
\end{align*}
where we have used the fact that the term $(f(x) - \mu_p)^2$ is bounded by $b^2$. By assumption, $\| f - \mu_p\|_{2,p} \geq \sqrt{c}$ and $\| f - \mu_p\|_{2,q} \geq \| f - \mu_q\|_{2,q} \geq \sqrt{c}$. Therefore, plugging these two bounds back into \eqref{eq:var-decomp},
\begin{align*}
    \left| \| f - \mu_p\|_{2,p} - \| f - \mu_q\|_{2,q}  \right| \leq b\left(\frac{b}{2\sqrt{c}} + 1 \right)\| p - q \|_1,
\end{align*}
which concludes the proof.
\end{proof}

\subsection{Main Results}

We are now ready to bound the estimation error of the different MDP components. The following Lemma does exactly this for a fixed number of tasks. Theorem \ref{th:rtp-err-models} extends this to the sequential setting.

\begin{lemma}\label{lemma:model-err-m}
Let $\tilde{p}_\theta(s' | s,a)$ and $\tilde{q}_\theta(u | s,a)$ be the transition and reward distributions estimated by Algorithm \ref{alg:hmm-learning} after $h$ tasks, $\tilde{V}_\theta^*$ be the optimal value functions of these models, and $\tilde{\sigma}^r_{\theta}(s,a)$, $\tilde{\sigma}^p_{\theta}(s,a; \theta')$ be the corresponding variances. Assume that the latter are both bounded below by some positive constant $\bar{\sigma} > 0$ for all $s,a,\theta,\theta'$. Then, with probability at least $1-\delta'$, the following hold simultaneously:
\begin{align*}
    |r_\theta(s,a) - \tilde{r}_\theta(s,a)| &\leq \rho_4(\Theta, T) U \sqrt{\frac{\log\left(6SA(S+U) / \delta'\right)}{h}}, \\
    |(p_\theta(s,a) - \tilde{p}_\theta(s,a))^T \tilde{V}_{\theta'}^*| &\leq \frac{\rho_4(\Theta, T)S}{1-\gamma} \sqrt{\frac{\log\left(6SA(S+U) / \delta'\right)}{h}}, \\ 
        |\sigma_\theta^r(s,a) - \tilde{\sigma}_{\theta}(s,a)| &\leq \left(\frac{1}{2\bar{\sigma}} + 1 \right) \rho_4(\Theta, T) U \sqrt{\frac{\log\left(6SA(S+U) / \delta'\right)}{h}}, \\
    |\sigma_\theta^r(s,a; \tilde{V}_{\theta'}^*) - \tilde{\sigma}_{\theta}(s,a; \tilde{V}_{\theta'}^*)| &\leq \left(\frac{1}{2\bar{\sigma}(1-\gamma)} + 1 \right) \frac{\rho_4(\Theta, T)S}{1-\gamma} \sqrt{\frac{\log\left(6SA(S+U) / \delta'\right)}{h}}.
\end{align*}
\end{lemma}
\begin{proof}
Recall that the estimated transition and reward probabilities are extracted from the columns of the observation matrix $\hat{O}$. Therefore, we can directly use Lemma \ref{lemma:O-error} to bound their error. For each state $s$, action $a$, next state $s'$, and reward $u$, we have
\begin{align*}
    |q_\theta(u | s,a) - \tilde{q}_\theta(u | s,a)| &\leq \rho_4(\Theta, T) \sqrt{\frac{\log\left(6SA(S+U) / \delta'\right)}{h}},\\
    |p_\theta(s' | s,a) - \tilde{p}_\theta(s | s,a)| &\leq \rho_4(\Theta, T) \sqrt{\frac{\log\left(6SA(S+U) / \delta'\right)}{h}}.
\end{align*}
These inequalities hold simultaneously with probability at least $1-\delta'$. Therefore, since rewards are bounded in $[0,1]$,
\begin{align*}
    |r_\theta(s,a) - \tilde{r}_\theta(s,a)| \leq \| q_\theta(\cdot | s,a) - \tilde{q}_\theta(\cdot | s,a) \|_1  &\leq \rho_4(\Theta, T) U \sqrt{\frac{\log\left(6SA(S+U) / \delta'\right)}{h}}.
\end{align*}
Similarly, for any function taking values in $[0, 1/(1-\gamma)]$,
\begin{align*}
    |(p_\theta(s,a) - \tilde{p}_\theta(s,a))\tr V| &\leq \frac{1}{1-\gamma}\| p_\theta(\cdot | s,a) - \tilde{p}_\theta(\cdot | s,a) \|_1 \leq \frac{\rho_4(\Theta, T)S}{1-\gamma} \sqrt{\frac{\log\left(6SA(S+U) / \delta'\right)}{h}}.
\end{align*}
Finally, a direct application of Lemma \ref{lemma:var-bound} yields the desired bounds on the variances,
\begin{align*}
    |\sigma_\theta^r(s,a) - \tilde{\sigma}^r_{\theta}(s,a)| &\leq \left(\frac{1}{2\bar{\sigma}} + 1 \right) \rho_4(\Theta, T) U \sqrt{\frac{\log\left(6SA(S+U) / \delta'\right)}{h}}, \\
    |\sigma_\theta^r(s,a; \theta') - \tilde{\sigma}^p_{\theta}(s,a; \theta')| &\leq \left(\frac{1}{2\bar{\sigma}(1-\gamma)} + 1 \right) \frac{\rho_4(\Theta, T)S}{1-\gamma} \sqrt{\frac{\log\left(6SA(S+U) / \delta'\right)}{h}}.
\end{align*}
\end{proof}

\rtpmodels*
\begin{proof}
Let $E_h$ be the event under which the bounds of Lemma \ref{lemma:model-err-m} all hold after observing $h$ tasks. We need to prove that, with high probability, there is no $h$ in which the event does not hold. We know that $E_h$ hold with probability at least $1 - \delta''$ from Lemma \ref{lemma:model-err-m}, where $\delta'' = \frac{6\delta'}{\pi^2 h^2}$. Too see this, notice that we introduced an extra $\pi^2h^2/6$ term in the confidence level $c_{h,\delta'}$. Then,
\begin{align*}
    \prob{\exists h \geq 1 : E_h = 0} \leq \sum_{h=1}^\infty \prob{E_m = 0} \leq \frac{6\delta'}{\pi^2}\sum_{h=1}^\infty \frac{1}{h^2} = \delta',
\end{align*}
where the first inequality is the union bound, the second is from Lemma \ref{lemma:model-err-m}, and the last equality is from the value of the $p$-series for $p=2$. Then, the main theorem follows after renaming the constants in Lemma \ref{lemma:model-err-m}.
\end{proof}

\begin{lemma}\label{lemma:prob-theta-init}
Let $\bar{\Theta}_h$ be a set that contains $\theta^*_h$ with probability at least $1-\delta$. Then, for any $\delta' \in (0,1)$ and $\theta\in\Theta$, with probability at least $1-\delta'$,
\begin{align}
    \prob{\theta^*_{h+1} = \theta} \leq \sum_{\theta'\in\bar{\Theta}_h} \hat{T}(\theta,\theta') + \delta k + \rho_T(\Theta, T)k \sqrt{\frac{\log 6SA(S+U)/\delta'}{h}}.
\end{align}
\end{lemma}
\begin{proof}
We start by bounding the probability that the next task is $\theta$ as
\begin{align*}
    \prob{\theta^*_{h+1} = \theta} &= \sum_{\theta'\in\Theta} \prob{\theta^*_{h+1} = \theta | \theta^*_{h} = \theta'}\prob{\theta^*_{h} = \theta'} \leq \sum_{\theta'\in\bar{\Theta}_h} \prob{\theta^*_{h+1} = \theta | \theta^*_{h} = \theta'}\prob{\theta^*_{h} = \theta'} + \delta k \\ &\leq \sum_{\theta'\in\bar{\Theta}_h} \prob{\theta^*_{h+1} = \theta | \theta^*_{h} = \theta'} + \delta k = \sum_{\theta'\in\bar{\Theta}_h} T(\theta,\theta') + \delta k,
\end{align*}
where the first inequality follows from the condition on $\bar{\Theta}_h$, the second by bounding the probability by one, and the last equality by definition of $T$. The result follows by applying Lemma \ref{lemma:err-T} and renaming the constants.
\end{proof}

\preelim*
\begin{proof}
We have
\begin{align*}
    \prob{\theta^*_h \notin \tilde{\Theta}_h} \leq \prob{\theta^*_h \notin \tilde{\Theta}_h \wedge \theta^*_{h-1} \in \tilde{\Theta}_{h-1}} + \prob{\theta^*_{h-1} \notin \tilde{\Theta}_{h-1}} \leq \sum_{l=2}^h \prob{\theta^*_l \notin \tilde{\Theta}_l \wedge \theta^*_{l-1} \in \tilde{\Theta}_{l-1}},
\end{align*}
where we applied the first inequality recursively to obtain the second one. We can bound each term in the second sum by
\begin{align*}
    \prob{\theta^*_l \notin \tilde{\Theta}_l \wedge \theta^*_{l-1} \in \tilde{\Theta}_{l-1}} &\leq \prob{\theta^*_l \notin \tilde{\Theta}_l \wedge \theta^*_{l-1} \in \tilde{\Theta}_{l-1} \wedge \theta^*_{l-1} \in \bar{\Theta}_{l-1}} + \prob{\theta^*_{l-1} \notin \bar{\Theta}_{l-1}}\\ &\leq \prob{\theta^*_l \notin \tilde{\Theta}_l \wedge \theta^*_{l-1} \in \tilde{\Theta}_{l-1} \wedge \theta^*_{l-1} \in \bar{\Theta}_{l-1}} + \delta,
\end{align*}
where $\bar{\Theta}_{l-1}$ is the set of active models after the application of Algorithm \ref{alg:policy-transfer} to $\tilde{\Theta}_{l-1}$ and the second inequality follows from the fact this algorithm discards $\theta^*_{l-1}$ with probability at most $1-\delta$. In order to bound the first term,
\begin{align*}
    \prob{\theta^*_l \notin \tilde{\Theta}_l \wedge \theta^*_{l-1} \in \tilde{\Theta}_{l-1} \wedge \theta^*_{l-1} \in \bar{\Theta}_{l-1}} \leq \sum_{\theta \in \Theta} \prob{\theta^*_l = \theta \wedge \theta \notin \tilde{\Theta}_l \wedge \theta^*_{l-1} \in \tilde{\Theta}_{l-1} \wedge \theta^*_{l-1} \in \bar{\Theta}_{l-1}}.
\end{align*}
Let us split this probability based on whether the concentration event on $T$ of Lemma \ref{lemma:err-T} holds or not. Using $\frac{\delta'}{3km^2}$ as confidence value, this probability is trivially bounded by $k\frac{\delta'}{3km^2}$ when the event of Lemma \ref{lemma:err-T} does not hold. Assume this event holds. Then, Lemma \ref{lemma:prob-theta-init} implies that the probability of the next task being $\theta$ is bounded by $\frac{\delta'}{3km^2}$, and the overall sum is bounded by $k\frac{\delta'}{3km^2}$. Notice that we divided $\delta'$ by $3km^2$ w.r.t. the value used in Lemma \ref{lemma:err-T} and Lemma \ref{lemma:prob-theta-init}. Putting these together,
\begin{align*}
    \prob{\theta^*_l \notin \tilde{\Theta}_l \wedge \theta^*_{l-1} \in \tilde{\Theta}_{l-1}} \leq \delta + \frac{2\delta'}{3m^2}.
\end{align*}
Using the union bound,
\begin{align*}
    \prob{\exists h \in [m] : \theta^*_h \notin \tilde{\Theta}_h} \leq \sum_{h=2}^m \prob{\theta^*_h \notin \tilde{\Theta}_h} \leq \sum_{h=2}^m \sum_{l=2}^h \prob{\theta^*_l \notin \tilde{\Theta}_l \wedge \theta^*_{l-1} \in \tilde{\Theta}_{l-1}} \leq \left(\delta + \frac{2\delta'}{3m^2}\right) m^2.
\end{align*}
Therefore, the result holds with probability at least $1-\delta'$ by taking $\delta \leq \frac{\delta'}{3m^2}$.
\end{proof}

\newpage

\section{Additional Details on the Experiments} \label{app:exp}

In this section, we provide all the necessary details to allow reproducibility of our experiments. We also provide further results and ablation studies to better understand how the proposed algorithms behave in different settings and with different parameters.

\subsection{Experiments with PTUM}

For these experiments we used a variant of the $4$-rooms domain of \citet{sutton1999between}. Here the agent navigates a grid with two rooms divided by a vertical wall in the middle and connected by a door. There are four actions (up, down, right, left) and there is a probability of $0.1$ that the action chosen by the agent fails and a random one is taken instead. The agent always starts in the bottom-left corner and must reach certain goal states. The reward function (e.g., the goal locations) and the size of the grid vary in each experiment and will be specified later.

\paragraph*{Experiments of Figure \ref{fig:known-models} \textit{left}}
The agent acts in a 12x12 grid and it receives reward of 1 when reaching a goal state, and 0 otherwise. There are twelve possible tasks with goals and doors in different positions and whose models are fully known to the agent. The parameters for running PTUM are: $\gamma=0.99$, $\delta=0.01$, $\epsilon=0.1$, and $n=100000$. Both Rmax and MaxQInit switch state-action pairs from unknown to known after 10 samples. Online algorithms without a generative model are run for $100$ episodes of $100$ steps each. Since goal states are modelled as absorbing states with reward 1, these algorithms are able to retrieve significant information in each episode to make the comparison with PTUM (which uses a generative model) fair. In fact, it suffices to reach a goal to get information about it. For this reason, each sample retrieved by PTUM from the generative model is considered as a single step in an episode.

\paragraph*{Experiments of Figure \ref{fig:known-models} \textit{right}}
Here the agent acts in a 12x12 grid without any wall, and has to discriminate among $7$ possible tasks. Each task has $7$ goals whose position is shared. However, given a goal state, reward values can differ between tasks.  To be more precise, goals are placed in the corners or nearby them. The true task has the optimal goal giving a reward of 0.8, while all the other tasks have a different best-goal with a reward of 0.81.
We set $\epsilon=1$, $\gamma=0.9999$, $\delta=0.01$, and $n=100000$. Both RMax and MaxQInit switch state-action pairs from unknown to known after 240 samples. The online algorithms without a generative model are run for 1000 episodes of 100 steps each. Due to the fact that goal states are modeled as absorbing state with reward 0 and, since all tasks have goals in the same position, online algorithms are able to get at most one informative sample in each episode (i.e., one with non-zero reward). For this reason, each sample retrieved by PTUM from the generative model is reported as an entire episode in the plots, which makes comparison fair.

It is important to note that in both these experiments the parameter needed by Rmax to switch between unknown and known state-action pairs (which directly affects the learning speed) is set way below the one recommended by the theory (see \citet{brafman2002r}).

\paragraph*{Additional results}

In order to verify the main theoretical results from the sample-complexity analysis of PTUM, we report ablation studies for its main parameters. 

\begin{figure*}[t]
\centering
\includegraphics[height=4cm]{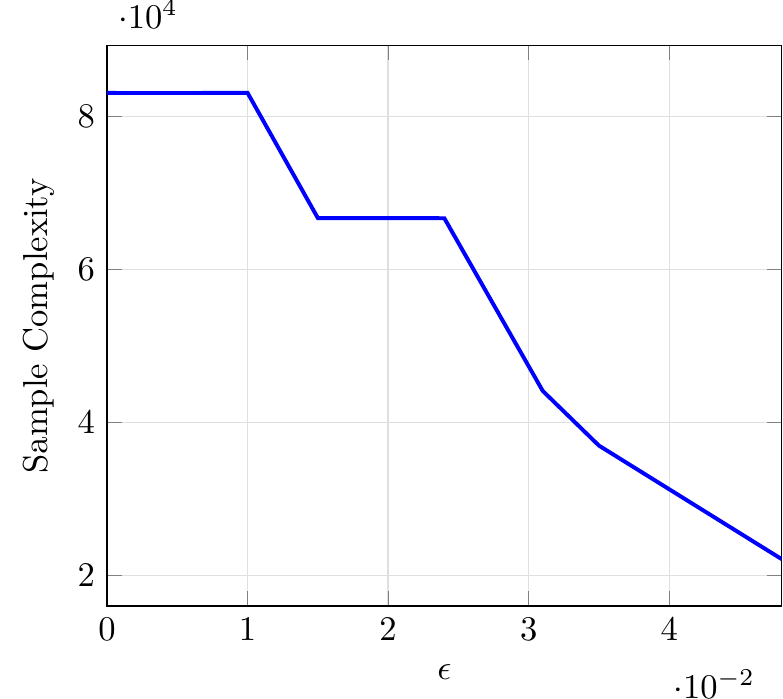}
\includegraphics[height=4cm]{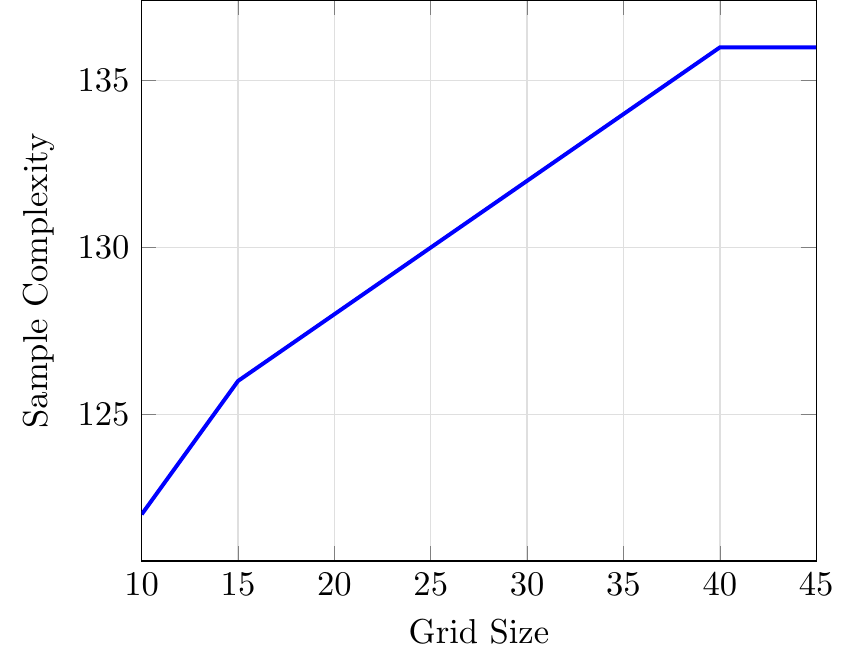}
\includegraphics[height=4cm]{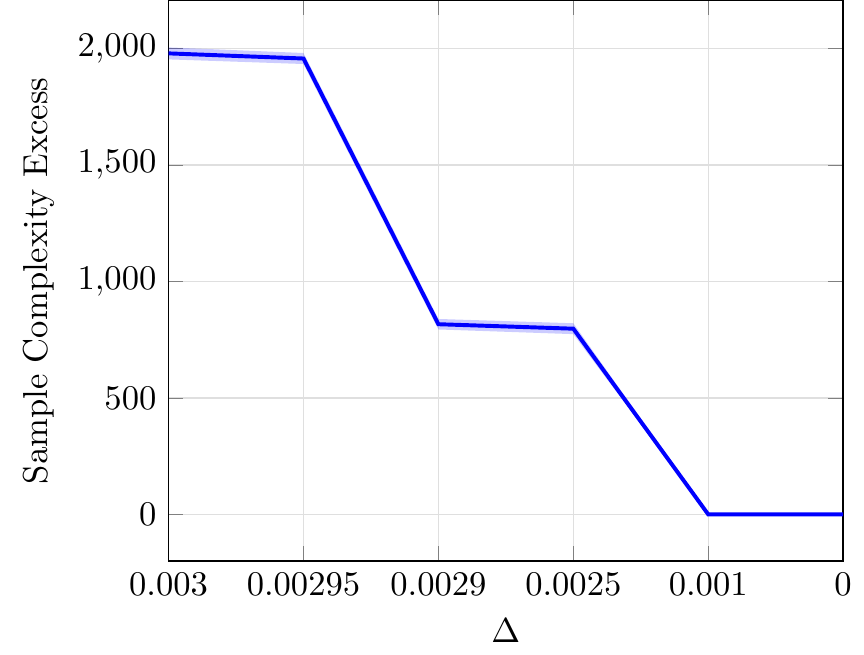}
\caption{Ablation studies for the main parameters of PTUM. (left) The sample complexity is a piece-wise constant and bounded  function of $\epsilon$. (middle) The sample complexity grows logarithmically with the grid size (i.e., the number of states). (right) The excess in sample complexity (w.r.t. the one with perfect models) decreases with the approximation error $\Delta$.}
\label{fig:ptum-ablation}
\end{figure*}

\paragraph*{Sample complexity vs accuracy}
The setting in this case is a grid of dimension 12x12 where a reward of 1 is given when reaching a goal state and 0 otherwise. 
The agent has knowledge of 12 possible perfect models, each of which has a single goal state located in the opposite corner from the starting one. The door position is different in each task. We set $\gamma$ to 0.99, $\delta$ to 0.01, $n$ to $1000000$, and test our algorithm for different values of $\epsilon$. 

Figure \ref{fig:ptum-ablation}(\textit{left}) shows how the sample complexity changes as a function of $\epsilon$. First, we notice that the function is piece-wise constant, which is a direct consequence of the finite set of models. In fact, varying $\epsilon$ changes the set of models $\Theta_\epsilon$ that must be discarded by the algorithm before stopping. Second, we notice that the function is bounded in $\epsilon$, i.e., we are allowed to set $\epsilon=0$ and the algorithm returns an optimal policy after discarding all the models different from the true one.

\paragraph*{Sample complexity vs number of states}
In this experiment, the agent faces grids of increasing sizes. The agent obtains reward 1 when ending up in a goal state and 0 otherwise. We consider three known models, each with the goal state in a corner different from the starting one and a different door position. We set $\epsilon$ to 0.0001, $\gamma$ to $0.99$, $\delta$ to $0.01$, and $n$ to $100000$. 

Figure \ref{fig:ptum-ablation}(\textit{middle}) shows how the sample complexity changes when the grid size (i.e., the number of states) increases. As expected, we obtain a logarithmic growth due to the union bounds used to form the confidence sets. As before, the function is piece-wise due to the finite number of models.

\paragraph*{Sample complexity vs model error}
A grid of fixed size 6x6 is considered. The agent has knowledge of 6 models, each of which has the goal state placed in the opposite corner w.r.t. the starting one. All models have reward 1 when reaching the goal state and 0 otherwise. Each task differs from the others in the door position. We set $\epsilon$ to $0.13$, $\gamma$ to $0.9$, $\delta$ to $0.1$, and $n$ to $1000000$. Here we study the sample complexity of PTUM with the exact set of models when varying the maximum uncertainty $\Delta$.

Figure \ref{fig:ptum-ablation}(\textit{right}) shows the excess in sample complexity w.r.t. the case with perfect models when the bound $\Delta$ on the approximation error decreases. Once again, we obtain a piece-wise constant function since, similarly to $\epsilon$, a higher error bound $\Delta$ changes the set of models that must be discarded by the algorithm and hence its sample complexity. Notably, we do not require $\Delta = 0$ to recover the "oracle" sample complexity.

\subsection{Sequential Transfer Experiments}

For these experiments we consider a grid-world similar to the objectworld domain of \citet{levine2011nonlinear}. Here we have an agent
navigating a $5\times 5$ grid where each cell is either empty or contains one item (among a set of possible items). Each item is associated to a different
reward value and can be picked up by the agent when performing a specific action (among up, down, left, right) in the state containing the
item. To keep the problem simple and Markovian, we suppose items immediately re-spawn after being picked up. That is, the agent can pick up
the same item indefinitely as far as it manages to return to the state containing it. 
The transition dynamics include a certain probability of failing the action as in the previous experiments. To
be more precise, in this setting failing actions regards exclusively the transition to other states, while the intended item is picked up with
another fixed probability. The agent sequentially faces $8$ different tasks as follows. Suppose the tasks are ordered in a list. Then, given the current task, with high probability the next task to be faced is the successor in the list, while there is a small probability of skipping one task and going two steps ahead or staying in the same task. The values of these probabilities (rewards, transitions, and task-transitions) will be specified in each experiment. For what concerns the estimation of the models, RTP is run with $100$ restarts for $100$ iterations. We use low-rank singular value decomposition to pre-process the empirical moments estimated by Algorithm \ref{alg:hmm-learning} as explained by \citet{anandkumar2014tensor}. This, empirically, seems to be somehow critical for the stability of the algorithm and for its computational efficiency. Regarding this last point, we note that after SVD is applied, Algorithm \ref{alg:hmm-learning} only needs to decompose a tensor of size $k\times k \times k$, where $k$ is the number of models ($k=8$ in this case).

\paragraph*{Experiments of Figure \ref{fig:sequential}}
As mentioned in the main paper, in this experiment only the reward changes between tasks, while the transition dynamics are fixed and have a failure probability of $0.1$. The failure probability of rewards is instead fixed to 0.012 between tasks. The possible objects that the agent can get have the following rewards: $\{0, 0.02, 0.04, 0.2, 0.22, 0.24, 0.5, 0.52, 0.54, 0.96, 0.98, 1\}$. Tasks are generated in the following way: first, task with indexes in $\{0, 2, 3, 5\}$ are randomly generated using only objects with values $\{0, 0.2, 0.5, 0.96\}$. Then tasks with index 1 and 7 are generated starting from model 0, in order for them to be $\epsilon$-optimal with it. The same is done for task with index 4 and 5 w.r.t. task with index 5. In particular, these modifications are carried out randomly, using objects with a slightly higher reward w.h.p. in each state-action pair. A small probability of using objects with smaller rewards is also present. We used this task-generation process in order to guarantee that there exist tasks which are hard to distinguish by PTUM, which makes the problem more challenging. If we simply generated tasks randomly as described before (without creating any ``similar" copies), the identification problem would become almost trivial even in the presence of estimated models.

The task-transitions succeed to the natural next task in the list with probability $0.97$, and fail in favor of the two adjacent tasks with probability $0.015$ each. A no-transfer uniform-sampling strategy is blindly run for $300$ tasks, querying $50$ samples per each state-action pair. This is done to make sure that the estimated models eventually become accurate enough to make PTUM enter the transfer mode. Once this $300$
tasks are over, the transfer begins and goes on for $100$ tasks. Once the model has been identified, a post-sample of $30$ queries is run in
each state-action pair in order to obtain a minimum accuracy level in the empirical models. When the identification of PTUM fails (i.e., the algorithm exceeds the budget $n$), no-transfer uniform sampling is run, and its complexity is reported in the plots.

For the computation of the model inaccuracy bounds (as prescribed by Theorem \ref{th:rtp-err-models}) we set $\rho = 0.135$ for all models. We chose this value so that the inaccuracies after the start-up phase are small enough to make the algorithm enter the transfer mode. Since we noticed that the estimated models become accurate rather quickly, we further decided to decay $\rho$ for the first $100$ tasks when PTUM enters the transfer mode from its initial value to $0.006$. This allows us to plot a faster and more clear transitory in the sample complexity, but in principle this step could be ignored as the model inaccuracies naturally decay by Theorem \ref{th:rtp-err-models}. For what concerns the computation of the initial sets of models for PTUM from the estimated task-transition matrix $\hat{T}$, we use the technique described in Theorem \ref{th:pre-elim} with $\eta = 0.087$ and $\rho_T = 0.001$. These values were chosen so that the algorithm is likely not to discard models whose predicted probability is above $0.005$. To add further robustness, we always keep the top-$3$ most probable next models (according to $\hat{T}$), even if they would be eliminated by the previous condition.

At each step, the agent is required to find a $\epsilon$-optimal policy with $\epsilon = 0.5$ and $\delta=0.01$. The discount factor $\gamma$ is set to 0.9. The expected returns of Figure \ref{fig:sequential} \textit{right} are estimated by running the obtained policies for $30$ episodes of $10$ steps each and averaging the results. Finally, the update constant for MaxQInit is fixed to 100.

\paragraph*{Changing both rewards and transition probabilities}

Since in our original domain only the positions of the items (i.e., the rewards) change across tasks, here we consider a variant where the
transition probabilities change as well. We now generate tasks as follows. The failure probability of the actions changes as well (ranging from
$0.1$ up to $0.45$), and we randomize the rewards as explained before. Again, tasks are generated to be $\epsilon$-optimal as described above,
and all the other parameters are set in the same way, except for the uniform sampling costant: no-transfer iterations query the generative
model for $150$ samples in each state-action pair, while the post-sampling (after identification) takes $100$ samples.

The results, shown in Figure \ref{fig:sequential-all}, are coherent with the ones of Figure \ref{fig:sequential} for the case where only rewards change. This is an interesting result since the number of parameters that must be estimated by the spectral learning algorithm doubled.

\begin{figure*}[t]
\centering
\includegraphics[height=4cm]{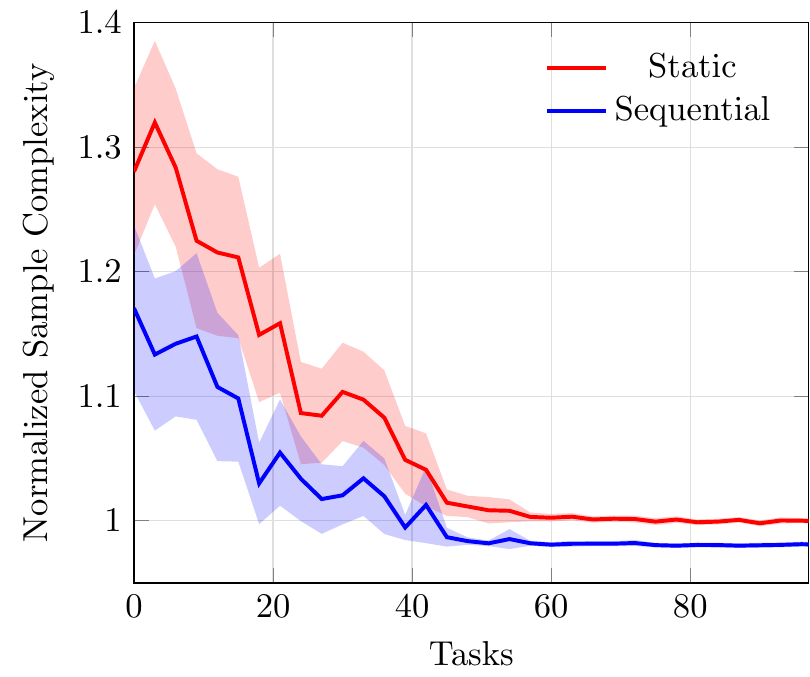}
\includegraphics[height=4cm]{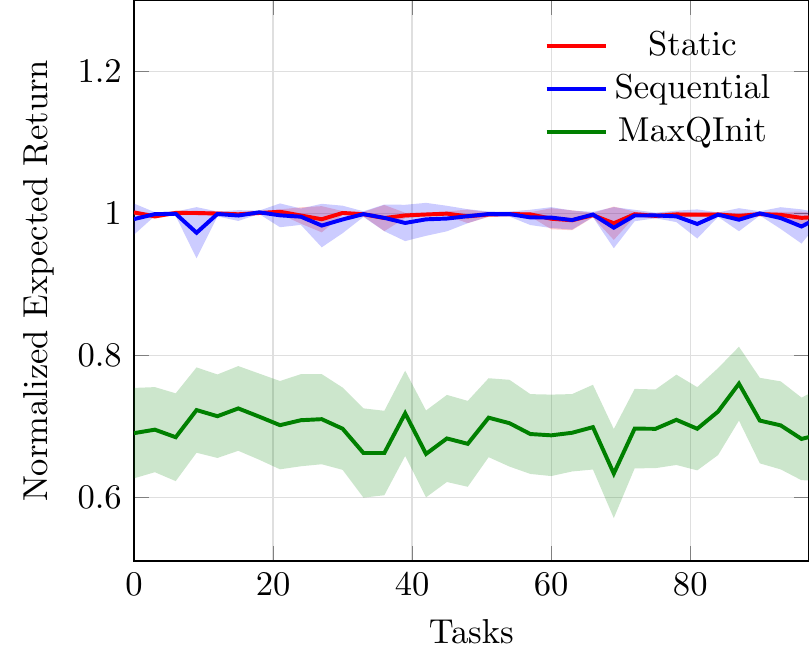}
\caption{Sequential transfer experiment when both rewards and transition probabilities change across tasks. (left) The sample complexity normalized by the one of PTUM with known models. (right) The expected return normalized by the optimal one for each specific task.}
\label{fig:sequential-all}
\end{figure*}

\paragraph*{Increasing the MDP stochasticity}

In the experiments reported above, both the rewards and the transition probabilities are almost deterministic. We now repeat these experiments
by increasing the stochasticity in both components. For what concerns the rewards, the possible objects are $\{0, 0.02, 0.2, 0.22, 0.5, 0.52,
0.96, 1\}$, and they are all used when generating the tasks. Moreover, the fact of generating $\epsilon$-optimal tasks is dropped here. The
failure probability of the reward is set to $0.3$. The transition failure probability is instead the same of the previous experiment, ranging from $0.1$ to $0.45$, depending on the task. Due to the higher stochasticity, no-transfer uniform sampling requires $250$ samples per each state-action pair, while transfer
post-sampling is set to $200$. All the other parameters are kept the same. 

Figure \ref{fig:sequential-stoch}(\textit{left}) shows the results. Once again, we have the sequential transfer algorithm outperforming its static counterpart. Compared to the previous experiments, here we notice only a small improvement of the sample complexity as a function of the number of tasks (i.e., rather flat curves). This is mostly due to the fact that now it is harder to have tasks that are very close to each other. This implies that the set of models that need to be discarded by PTUM remains roughly the same after the algorithm starts entering the transfer mode, and thus improvements in the model accuracy only lead to small improvements in the sample complexity.

\begin{figure*}[t]
\centering
\includegraphics[height=4cm]{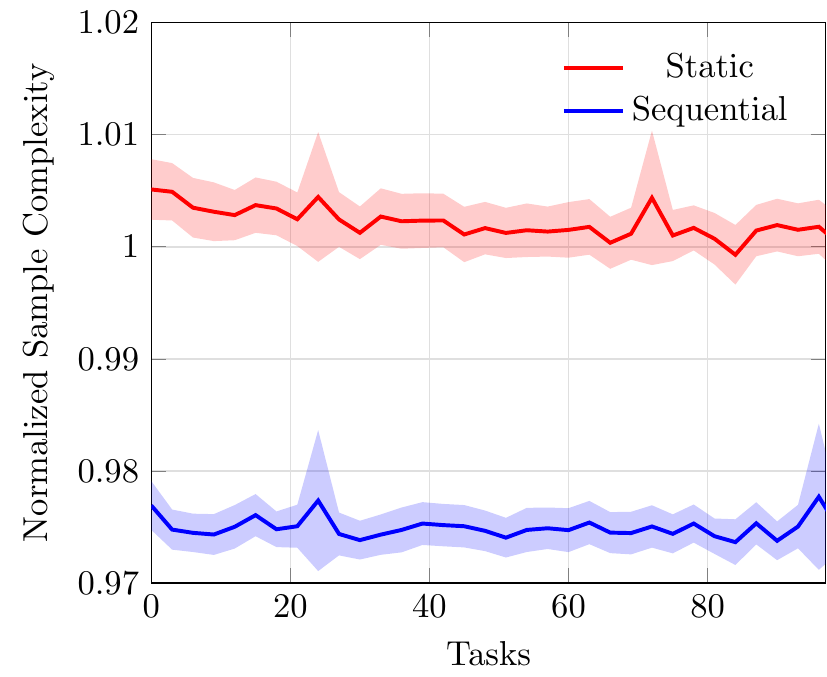}
\includegraphics[height=4cm]{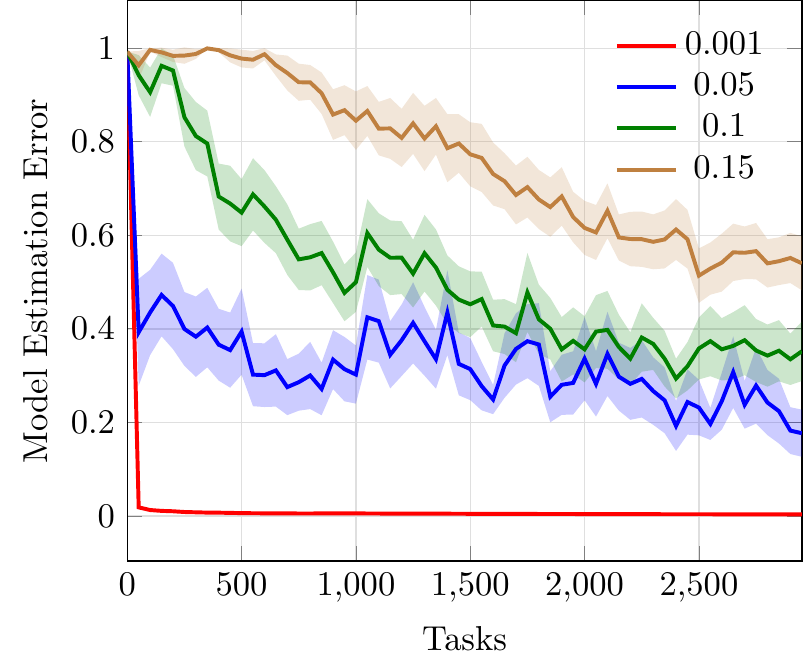}
\caption{(left) Sequential transfer experiment with higher reward and transition stochasticity. (right) Sequential transfer experiment with higher stochasticy in the transition between tasks. Each curve corresponds to a different value of failure probability in the task-transition matrix.}
\label{fig:sequential-stoch}
\end{figure*}

\paragraph*{Increasing the task-transition stochasticity}

In the previous experiments the transitions between tasks are almost deterministic. The stochasticity in these components seems to be
the most critical parameter for what concerns the number of tasks needed by RTP to guarantee accurate estimations. For this purpose, we now show
the $l_\infty$ error of the estimated MDP models (with respect to the true ones) for different levels of stochasticity in the task-transition probabilities. Here we fix $10$ different tasks, generated by randomizing the items as explained before, and run the RTP algorithm sequentially to obtain the error estimates.

Figure \ref{fig:sequential-stoch}(\textit{right}) shows the results. Here we clipped the maximum error to 1 for better visualization. This is anyway reasonable since the algorithm estimates probabilities and we could alternatively normalize the estimates. We note that the estimation error decays as expected for all values of stochasticity. The algorithm requires, however, many more tasks to get accurate estimates when the stochasticity increases.

\end{document}